\pdfoutput=1
\documentclass[10pt]{article}

\usepackage{times}
\usepackage{graphicx}
\usepackage{amsmath}
\usepackage{amssymb}
\usepackage{bm}
\usepackage{bbm}
\usepackage{algorithm}
\usepackage{algorithmic}
\usepackage[toc,page]{appendix}

\usepackage{array}
\newcolumntype{?}{!{\vrule width 1pt}}
\newcolumntype{C}[1]{>{\centering\arraybackslash\hspace{0pt}}p{#1}}

\newcounter{nbdrafts}
\makeatletter
\newcommand{\checknbdrafts}{
\ifnum \thenbdrafts > 0
\@latex@warning@no@line{**********************************************************************}
\@latex@warning@no@line{* The document contains \thenbdrafts \space draft note(s)}
\@latex@warning@no@line{**********************************************************************}
\fi}

\makeatother

\newcommand{\comment}[1]{}
\newcommand{\OURSR}[0]{{\bf OURS-RANDOM}}

\newcommand{\OURSM}[0]{{\bf OURS-MAXW}}
\newcommand{\BASEM}[0]{{\bf BASELINE-MAXW}}

\newcommand{\bX}[0]{\mathbf{X}}
\newcommand{\bx}[0]{\mathbf{x}}
\newcommand{\bI}[0]{\mathbf{I}}

\newcommand{\mC}[0]{\mathcal{C}}
\newcommand{\mE}[0]{\mathcal{E}}
\newcommand{\mH}[0]{\mathcal{H}}

\newcommand{\mL}[0]{\mathcal{L}}
\newcommand{\mX}[0]{\mathcal{X}}

\newcommand{\mO}[0]{\mathcal{O}}
\newcommand{\mN}[0]{\mathcal{N}}

\newcommand{\KL}[0]{\text{KL}}

\newif\ifdraft
\drafttrue
\draftfalse

\ifdraft
 \newcommand{\PF}[1]{\textcolor{blue}{{\bf PF: #1}}}
 
 \newcommand{\PB}[1]{\textcolor{red}{{\bf PB: #1}}}
 \newcommand{\pb}[1]{\textcolor{red}{#1}}
  
 \newcommand{\FF}[1]{\textcolor{red}{{\bf FF: #1}}}
 
\else
 \newcommand{\PF}[1]{}
 \newcommand{\FF}[1]{}
 \newcommand{\PB}[1]{}

 \newcommand{\pb}[1]{ #1 }
 
\fi
\newcommand{\sttt}[1]{{\footnotesize{\texttt{#1}}}}

\newtheorem{theorem}{Theorem}[section]
\newtheorem{lemma}[theorem]{Lemma}
\newenvironment{proof}[1][Proof]{\begin{trivlist}
\item[\hskip \labelsep {\bfseries #1}]}{\end{trivlist}}

\begin{document}

\title{Multi-Modal Mean-Fields via Cardinality-Based Clamping}

\author{Pierre Baqu{\'{e}}$^1$ \quad\quad Fran{\c{c}}ois Fleuret$^{1,2}$ \quad\quad Pascal Fua$^1$ \\
$^1$CVLab, EPFL, Lausanne, Switzerland\\
$^2$IDIAP, Martigny, Switzerland\\
{\tt\small \{firstname.lastname\}@epfl.ch}
}

\maketitle

\begin{abstract}

  Mean Field inference is central to statistical physics.  It has attracted much
  interest  in  the Computer  Vision  community  to efficiently  solve  problems
  expressible in terms of large Conditional  Random Fields.  However,
  since it models the posterior probability distribution as a product of marginal probabilities,
  it may fail to properly account for important dependencies between variables.

  We therefore replace the fully  factorized distribution of  Mean Field
  by  a weighted  mixture of  such distributions,  that similarly  minimizes the
  KL-Divergence to  the true posterior.   By introducing two new  ideas, namely,
  conditioning  on groups  of  variables  instead of  single  ones  and using  a
  parameter of  the conditional random field  potentials, that we identify  to the
  temperature in the sense of statistical  physics to select such groups, we can
  perform this minimization  efficiently.  Our extension of the clamping method
  proposed  in previous  works  allows us  to both  produce  a more  descriptive
  approximation of the true posterior and, inspired by the diverse MAP paradigms,
  fit a mixture of Mean Field approximations.  We demonstrate that this
  positively impacts real-world algorithms that initially relied on mean fields.

\end{abstract}

\vspace{-0.2cm}

\section{Introduction}

Mean Field (MF)  is a   modeling technique  that has been  central to
statistical  physics for  a  century.  Its ability  to handle  stochastic
models  involving millions  of variables  and dense
graphs has attracted much attention in our community.  It is routinely
used for tasks as diverse as detection~\cite{Fleuret08a,Bagautdinov15}, segmentation~\cite{Saito12,Kraehenbuehl13,Chen15b,Zheng15}, denoising~\cite{Cho00,Nowozin11,Li14},
depth from stereo~\cite{Fransens06a,Kraehenbuehl13} and pose-estimation~\cite{Vineet13}.

MF approximates  a ``true'' probability  distribution by a  fully-factorized one
that is easy to encode and manipulate~\cite{Koller09}.  The true distribution is usually defined
in  practice  through  a  Conditional  Random   Field  (CRF),  and  may  not  be
representable  explicitly, as  it  involves  complex inter-dependencies  between
variables. In such a case the MF approximation is an extremely useful tool.

While  this drastic  approximation often  conveys the  information of  interest\comment{,
usually the  marginal distributions}, the  true distribution may concentrate on
configurations  that are  very different,  equally  likely, and  that cannot  be jointly
encoded by a product law.  Section~\ref{sec:motivation} depicts such a case
where groups of variables are correlated and may take one among many values with
equal  probability.   In   this  situation,  MF  will  simply   pick  one  valid 
configuration, which we call a mode,  and  ignore   the  others.    So-called  structured   Mean Field
methods~\cite{Saul95,Bouchard-Cote09} can  help overcome this  limitation.  This
can be  effective but requires arbitrary  choices in the design  of a simplified
sub-graph  for each  new problem,  which can  be impractical  especially if  the
initial CRF is very densely connected.

Here  we  introduce  a novel  way  to  automatically  add  structure to  the  MF
approximation and  show how it can  be used to return  several potentially valid
answers  in  ambiguous  situations.   Instead  of  relying  on  a  single  fully
factorized   probability  distribution,   we   introduce  a   mixture  of   such
distributions, which  we will refer to  as {\it Multi-Modal Mean  Field} (MMMF).

We compute  this MMMF by  partitioning the state space  into subsets in  which a
standard MF  approximation suffices. This is  similar in spirit to  the approach
of~\cite{Weller15} but a key difference is that our clamping acts simultaneously
on arbitrarily sized groups of variables, as  opposed to one at a time. We will
show that  when dealing  with large CRFs  with strong  correlations,
this is essential. The key to the efficiency of MMMF is how we choose
these groups.  To this end, we  introduce a temperature parameter  that controls
how much we  smooth the original probability distribution  before the
MF  approximation. By  doing  so for  several temperatures,  we  spot 
groups of variables that may take different labels in different modes of the distribution.
We then force  the optimizer to explore alternative solutions  by clamping them,
that is, forcing them to take different values.  Our temperature-based
approach, unlike the one of~\cite{Weller15}, does not require  {\it a
  priori} knowledge of the CRF structure and is therefore compatible with ``black box'' models.

In  the remainder  of the  paper, we  will  describe both  MF and  MMMF in  more
details. We will then demonstrate that MMMF outperforms both MF and the clamping
method of~\cite{Weller15} on a range of tasks. 

\comment{ We will then show that it compares favorably to the one
of~\cite{Weller15}, on  the same  benchmark problems.  Finally, we  will replace
standard      MF      by      MMMF      in      a      pedestrian      detection
algorithm~\cite{Fleuret08a,Berclaz11}        and         two        segmentation
algorithms~\cite{Chen15b,Yu16b} that  rely on  dense CRFs~\cite{Kraehenbuehl13}.
In all  threes cases, we  will demonstrate that  MMMF allows better  handling of
ambiguous  situations  in  individual  frames and  to  leverage  temporal
consistency across frames to resolve ambiguities.
}

\comment{
We  will refer  to our  approach  as {\it  Multi-Modal Mean  Field} (MMMF).   It
combines the following contributions:
\vspace{-0.2cm}
\begin{itemize}
\setlength\itemsep{0em}

\item  We propose  {\it cardinality  clamping},  which makes  the clamping  idea
  of~\cite{Weller15}  efficient  for  real-size CRFs  with  strong  dependencies
  between variables.

\item We propose  to vary the temperature  parameter in a CRF in  order to chose
  the  best  variables to  clamp.  In  particular,  we  show that  the  critical
  temperature for  the dense Gaussian CRF~\cite{Kraehenbuehl13}  can be computed
  in closed form.
  
\item We extend the notion of diverse MAPs to a mixture on Mean Fields.
\end{itemize}
}



\section{Background and Related Work}

Conditional Random  Fields (CRFs)  are    often    used   to    represent   correlations    between
variables~\cite{Wang13c}. Mean Field inference is a means to approximate them
in a computationally efficient way. We briefly review both techniques below.

\subsection{Conditional Random Fields}
\label{sec:relatedCRF}

Let  $\bX  =  \left(X_1,  \ldots, X_N\right)$  represent  hidden  variables  and
$\bI$ an image  evidence.  A CRF  relates   the  ones  to the  others  via  a
posterior probability distribution
\comment{
\begin{align}
\vspace{-0.1cm}
P(\bX \mid \bI) &= \exp\left(-\mE(\bX \mid \bI) - \log(Z
  (\bI)) \right)  \label{eq:related:mrf}  \\
  & = \exp\left(\sum_{ c \in \mC}\bm{\phi}_c(\bX_c \mid \bI) - \log(Z
(\bI)) \right) \; \; , \nonumber
\end{align}
}
\begin{align}
\vspace{-0.2cm}
P(\bX \mid \bI) &= \exp\left(-\mE(\bX \mid \bI) - \log(Z
  (\bI)) \right)  \label{eq:related:mrf}  \;, 
\vspace{-0.1cm}
\end{align}
where $\mE(\bX \mid \bI)$ is an energy  function that is the sum of 
terms known as potentials $\bm{\phi}_c(\cdot)$ defined on a set of graph cliques
$c \in  \mC$, $\log(Z(\bI))$ is  the log-partition function that  normalizes the
distribution.
From now on, we  will omit the dependency with respect to  $\bI$.

\subsection{Mean Field Inference}
\label{sec:relatedMF}

The set  of all  possible configurations  of $\bX$, that we denote by $\mX$, is exponentially large, 
which makes the explicit  computation of marginals,
Maximum-A-Posteriori (MAP)  \pb{or $Z$} intractable and a  wide range of
variational methods have been proposed to approximate $P(\bX)$~\cite{Kappes15}.
Among those, Mean Field (MF) inference is one of the most popular. It involves
introducing a distribution $Q$ written as
\begin{equation}
\vspace{-0.2cm}
Q(\bX=(x_1,\ldots,x_N)) = \prod_{i = 1}^{N} q_i(x_i)\; ,
\vspace{-0.1cm}
\label{eq:fullyFactorized}
\end{equation}
where $q_i( \, .  \, )$ is a categorical discrete distribution defined
for $x_i$ in  a possible labels space $\mL$. The $q_i$ are estimated  by minimizing the
KL-divergence
\begin{equation}
\KL(Q || P) = \sum_{\bx \in \mX} Q(\bX = \bx ) \log \frac{Q(\bX = \bx )} {P(\bX = \bx )}\; .
\label{eq:related:kl}
\end{equation}
Since  $Q$  is fully  factorized,  the  terms  of  the $\KL$-divergence  can  be
recombined as a sum of an expected energy, containing as many terms as there are
potentials  and a  convex negative  entropy  containing one  term per  variable.
Optimization can then be performed  using a provably convergent gradient-descent
scheme~\cite{Baque16}.  \comment{ Sometimes, the distribution  itself is the desired final
result.   In others,  it  is used  to  compute a  MAP  assignment.}

As  will  be  shown  in  Section~\ref{sec:motivation},  this  simplification
  sometimes comes  at the cost  of downplaying  the dependencies  between variables.
  The {\it  DivMBest} method~\cite{Ramakrishna12b,Batra13} addresses  this issue
  starting from the  following observation: When looking for an  assignment in a
  graphical model,  the resulting MAP  is not  necessarily the best  because the
  probabilistic  model may  not capture  all that  is known  about the  problem.
  Furthermore, optimizers can get stuck  in local minima.  The proposed solution
  is to  sequentially find several local  optima and force them  to be different
  from  each  other  by  introducing  diversity  constraints  in  the  objective
  function.
It has
recently  been shown  that it  is
provably  more  effective to  solve for diverse MAPs  jointly but  under  the  same set  of
constraints~\cite{Kirillov15}. However, none of these methods provide a generic
and  practical way  to choose  local constraints  to be  enforced over  variable
sub-groups.  Furthermore, they only return a  set of MAPs.  By contrast, our
  approach  yields a  multi-modal approximation  of the  posterior distribution,
  which is a much richer description and which we will show to be useful.


\comment{
Another  approach to  improving  the  MF approximation,  is  to  clamp a  binary
variable $\bx_i$  to either  0 or  1 and to  approximate the  partition function
twice,  once   for  each   value
}

Another approach to improving the MF approximation is to decompose it into a
mixture of product laws by ``clamping'' some of the variables to fixed values, and
finding for each set of values the best factorized distribution under the resulting 
deterministic conditioning.   By  summing   the  resulting
approximations  of  the  partition  function,   one  can  provably  improve  the
approximation  of the  true  partition  function~\cite{Weller15}.  This  procedure  can then  be
repeated iteratively by  clamping successive variables but is only practical
for relatively small CRFs.  At  each iteration, the
variable to be  clamped is chosen on  the basis of the  graphical model weights,
which  requires intimate  knowledge about  its  internals, which  is not  always
available.

Our own  approach is in the  same spirit but  can clamp multiple variables  at a
time   without   requiring   any   knowledge   of   the   graph   structure   or
weights.

Finally, {\it  DivMBest} approaches do not provide a way to choose the best solution
without looking at the ground-truth, except  for the one  of~\cite{Yadollahpour13}
that  relies on training  a new classifier for that purpose.
By contrast, we  will  show that the multi modal Bayesian
nature of our output induces a principled  way  to  use temporal
consistency to solve directly practical problems.



\section{Motivation}
\label{sec:motivation}

To  motivate our  approach, we  present here  a toy  example that  illustrates a
typical failure  mode of  the standard  MF technique, which  ours is  designed to
prevent. Fig.~\ref{fig:schematic_approach} depicts  a CRF where each  pixel represents 
a binary variable connected to its neighbors by attractive pairwise potentials.

For the sake of illustration, we split  the grid into four zones as follows. The
attractive terms are  weak on left side  but strong on the  right. Similarly, in
the top part, the unary terms favor value of $1$ while being completely random in
the bottom part.

The    unary    potentials    are    depicted     at    the    top    left    of
Fig.~\ref{fig:schematic_approach} and  the result of the standard  MF approximation
at the bottom in terms of the probability of the pixels being assigned the label
$1$. In the bottom right corner of the grid,  because the interaction potentials are strong,
all pixels end up being assigned high probabilities of being 1 by MF, where they could
just as well  have all been assigned  high probabilities to be  zero. We explain
below how  our MMMF algorithm  can produce {\it  two} equally likely  modes, one
with all pixels being zero with high probability and the other with all pixel being one with high probability.


\begin{figure}[ht!]
\vspace{-0.3cm}
\begin{center}
\includegraphics[clip, width=0.8\textwidth]{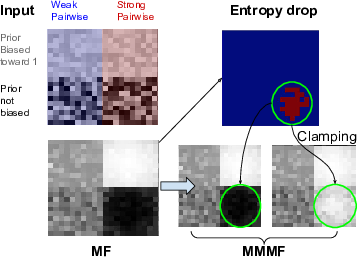} \\
\end{center}
\vspace{-0.3cm}
   \caption{A typical failure mode of MF resolved by MMMF. Grey levels indicate marginal probabilities, under the prior (Input) and under the product laws (MF and MMMF).} 
\label{fig:schematic_approach}
\vspace{-0.3cm}
\end{figure}
\vspace{-0.2cm}


\section{Multi-Modal Mean Fields}
\label{sec:MMMF}

Given  a CRF  defined with  respect  to a  graphical model  and the  probability
$P(\bX= \bx)$   for  all   states   in   $\mX$,  the   state   space  introduced   in
Section~\ref{sec:relatedCRF}, the standard MF approximation only models a single
mode  of the  $P$,  as discussed  in  Section~\ref{sec:relatedMF}. We  therefore
propose to create  a richer representation that accounts  for potential multiple
modes     by    replacing     the    fully     factorized    distribution     of
Eq.~\ref{eq:fullyFactorized} by  a weighted  mixture of such  distributions that
better minimizes the KL-divergence to $P$.

The  potential  roadblock  is  the  increased  difficulty  of  the  minimization
problem. In this section,  we present an overview of our  approach to solving it,
and discuss its key aspects in the following two.


Formally, let  us assume that  we have  partitioned $\mX$ into  disjoint subsets
$\mX_k$ for $1 \leq k \leq K$.  We replace the original Mean Field (MF) approximation
by one of the form
\begin{eqnarray}
P(\bX = \bx) \ \approx \ Q_{MM}(\bX= \bx) & \!\!\!\!=\!\!\!\! & \sum_k m_k Q_k(\bx) \;,
\label{eq:MultMeanField}
 \\
Q_k(\bx)                                  & \!\!\!\!=\!\!\!\! & \prod_i{q_i^k(x_i)}\;, \nonumber
\end{eqnarray}
where $Q_k$ is  a MF approximation  for the states $\bx  \in \mX_k$ with
individual probabilities $q_i^k$ that variable $i$ can take value $x_i$ in a set
of labels $\mL$, and  $m_k$ is the probability that a  state belongs to $\mX_k$.

 We can evaluate  the $m_k$ and
$q_i^k$ values by minimizing the KL-divergence between $Q_{MM}$ and $P$.  The key to
making this  computation tractable is  to guarantee  that we can  evaluate the $q_i^k$
parameters on each  subset separately by performing a  standard MF approximation
for each.   One way  to achieve  that is to  constrain the  support of  the $Q_k$
distributions to be disjoint, that is,
\begin{eqnarray}
\forall k \neq k^{\prime}, Q_{k^{\prime}}\left(\mX_k \right) & = & 0 \;.
\label{eq:MultMeanFieldConstraint}
\end{eqnarray}
In other words,  each MF approximation is {\it  specialized} on a subset
$ \mX_k$ of the state space and is computed to minimize the KL-Divergence there.
In practice, we enrich our approximation by recursively splitting a set of
states $\mX_k$ among our partition  $\mX_1,\dots,\mX_{K}$ 
into  two subsets $\mX^1_k$ and  $\mX^2_k$ to obtain the 
new partition  $\mX_1,\dots,\mX_{k-1},\mX^1_k,\mX^2_k,\mX_{k+1},\dots,\mX_{K}$,
 which is then reindexed from $1$ to $K+1$. Initially, $\mX_k$
represents  the whole  state space.  Then we  take it  to be  the newly  created
subset in a  breadth-first order  until a  preset number  of subsets  has been
reached. Each time,
the algorithm proceeds through the following steps:
\begin{list}{\labelitemi}{\leftmargin=1.25em}
\itemsep0em

 \item It finds groups of variables  likely to have different values in different
   modes of the distribution using an entropy-based criterion for the $q_i^k$.

 \item It  partitions the set  into two disjoint  subsets according to a clause
   that sets a threshold on the number of variables in this group that take a specific  label.
    $\mX^1_k$ will contain the states among $\mX_k$ that meet this clause and
    $\mX^2_k$ the others.

 \item  It performs  an MF  approximation  within each  subset independently  to
   compute  parameters $q_i^{k,1}$  and  $q_i^{k,2}$  for each  of
   them. This is done by a standard MF approximation, to which we add
   the disjointness constraint~\ref{eq:MultMeanFieldConstraint}.

\end{list}
This  yields a  binary tree  whose leaves  are the  $\mX_k$ subsets  forming the
desired   state-space   partition.  Given this partition, we can finally evaluate  the   $m_k$.    
In Section~\ref{sec:partioningStateSpace},  we  introduce   our  cardinality  based
criterion and show that it makes minimization of the KL-divergence possible.  In
Section~\ref{sec:clamp}, we  show how our entropy-based criterion selects, at each
iteration, the groups of variables on which the clauses depend.

\section{Partitioning the State Space}
\label{sec:partioningStateSpace}
In this section, we describe the cardinality-based criterion we use to
recursively split state spaces and  explain why it allows efficient optimization
of  the  KL-divergence   $\KL(Q_{MM}\|P)$,  where  $Q_{MM}$  is   the  mixture  of
Eq.~\ref{eq:MultMeanField}.
\subsection{Cardinality Based Clamping}
\label{sec:cardinalityClamping}
The state space partition  $\mX_{k \; , \; {1 \leq k  \leq K}}$ introduced above
is at the heart  of our approximation and its quality and tractability critically depend on how
well chosen it is.
In~\cite{Weller15}, each split  is obtained by clamping to zero  or one the value
of a single binary  variable.  In other words, given a set  of states $\mX_k$ to
be split, it is broken into subsets $\mX_k^1  = \{\bx \in \mX_k | x_i = 0\}$ and
$\mX_k^2 =  \{\bx \in  \mX_k | x_i  = 1\}$,  where $i$ is  the index  of a
specific variable. To compute a Mean Field approximation to $P$ on each of these
subspaces, one only  needs to perform a standard  Mean Field approximation while
constraining the $q_i$ probability assigned to the clamped variable to be either
zero  or  one.  \comment{Furthermore,  it  can  be  formally  proved that  this  approach
guarantees an improved approximation~\cite{Weller15}.}  However, this is limiting
for the large and dense CRFs used in practice because clamping only one variable
among  many at  a  time may  have  very little  influence  overall. Pushing  the
solution  towards  a  qualitatively  different minimum  that  corresponds  to  a
distinct mode may require simultaneously clamping many variables.

To remedy  this,
 we retain the  clamping idea but apply it to  groups of variables instead
of  individual ones  so as  to find  new modes  of the  posterior while  keeping the
estimation of the parameters $m_k$ and $q^k_i$ computationally tractable.
More specifically, given a  set of states $\mX_k$ to be split,  we will say that
the split into $\mX^1_k$ and $\mX^2_k$ is cardinality-based if
\begin{align}
 &  \mX^1_k  = \{\bx  \in  \mX_k  \text{  s.t. }  \sum  \limits_{u  = 1  \dots  L}
  \mathbbm{1}(\bx_{i_u} = v_u) \geq C \} \;, \label{eq:cardSplit1}\\
  & \mX^2_k = \{\bx \in \mX_k  \text{ s.t. }  \sum \limits_{u = 1  \dots L}
  \mathbbm{1}(\bx_{i_u} = v_u) < C \} \;, \label{eq:cardSplit2}
\end{align}
where the  $i_1,\dots,i_L$ denote groups  of variables that  are chosen
by the entropy-based criterion and $v_1,\dots,v_L$ is  a set of  labels in
$\mL$. In other words, in one of the  splits, more than $C$ of the variables have
the assigned values and in the other less than $C$ do. For example, for semantic
segmentation $\mX^1_k$  would be the set  of all segmentations in $\mX_k$ for  which at least
$C$ pixels  in a region  take a  given label, and  $\mX^2_k$ the set  of all
segmentations for which less than $C$ pixels do.

We will refer  to this approach as {\it cardinality  clamping}
  and will propose  a practical way  to select
  appropriate   $i_1,\dots,i_L$   and   $v_1,\dots,v_L$  for   each   split   in
  Section~\ref{sec:clamp}.

\subsection{Instantiating the Multi-Modal Approximation}
\label{sec:instatiatingMMMF}

The  {\it cardinality  clamping}  scheme introduced  above yields  a state  space
partition $\mX_{k \;  , \; {1 \leq k  \leq K}}$. We now show that  given such a
partition, minimizing the KL-divergence $\KL(Q_{MM} \| P)$ using the multi-modal
approximation  of  Eq.~\ref{eq:MultMeanField}  under  the  disjointness
constraint, becomes tractable.

In practice, we relax the constraint~\ref{eq:MultMeanFieldConstraint} to {\it near} disjointness
\begin{equation}
\forall k \neq k^{\prime}, Q_{k^\prime} \left(\mX_k \right) \leq \epsilon
\;,
\label{eq:MultMeanFieldConstraintEpsilon}
\end{equation}
where $\epsilon$  is a small  constant.  It makes the  optimization problem
better behaved and removes the need to tightly  constrain any individual
variable,  while retaining  the  ability  to compute  the  KL  divergence up  to
$\mathcal{O}(\epsilon       \log(\epsilon))$.

Let $\hat{m}$ and $\hat{q}$ stand for all the $m_k$ and $q_i^k$ parameters that
appear in Eq.~\ref{eq:MultMeanField}.  We compute them as
\begin{small}
  \begin{eqnarray}
\min \limits_{\hat{m},\hat{q}}\KL(Q_{MM} \| P)\hspace{-0.3cm} &=& \hspace{-0.2cm}\min \limits_{\hat{m},\hat{q}} \sum \limits_{\bx \in
  \mX} \sum \limits_{k \leq K}  m_k Q_{k}(\bx) \log \left( \dfrac{Q_{MM}(\bx)}{P(\bx)} \right) \nonumber \\
&\equiv& \hspace{-0.2cm} \min \limits_{\hat{m}}  \sum \limits_{k \leq K} m_k\log(m_k)
- \sum \limits_{k \leq K} m_k A_k \; ,\label{method:eq:A_k} \\
  \mbox{where \hspace{0.8cm}}  A_k & =  & \max  \limits_{q^k_i,{i = 1  \dots N}}
  \sum      \limits_{\bx      \in      \mX}     Q_{k}(\bx)      \log      \left(
  \dfrac{e^{-E(\bx)}}{Q_{k}(\bx)}\right) \; \label{method:eq:max_A_k}
  \end{eqnarray}
\end{small}
where $A_k$ is maximized under the near-disjointness constraint of Eq.~\ref{eq:MultMeanFieldConstraintEpsilon}.

%
%
As  proved  formally in  the  supplementary  material,  the second  equality  of
Eq.~\ref{method:eq:A_k} is valid up to a constant and after neglecting a term of
order $\mathcal{O}(\epsilon \log  \epsilon)$ which appears under  the {\it near}
disjointness   assumption  of   the  supports.    Given  the   $A_k$  terms   of
Eq.~\ref{method:eq:max_A_k}  and   under  the   constraints  that   the  mixture
probabilities $\hat{m}$ sum to one, we must have
\begin{equation}
\label{method:eq:mk}
m_k = \dfrac{e^{A_k}}{\sum \limits_{k^\prime \leq K} e^{A_{k^\prime}}} \;,
\end{equation}
and we  now turn to the  computation of these  $A_k$ terms.  We formulate  it in
terms of a constrained optimization problem as follows.

\subsubsection{Handling Two Modes}
\label{sec:twoModes}

\comment{For simplicity's  sake,} Let us first  consider the case where  we generate
only  two modes  modeled  by $Q_1(\bx)=\prod  q^1_{i}(x_i)$ and  $Q_2(\bx)=\prod
q^2_{i}(x_i)$ and  we seek to  estimate the $q^1_{i}$  probabilities.
The $q^2_{i}$ probabilities are evaluated similarly.

Recall from Section~\ref{sec:instatiatingMMMF} that  the  $q^1_{i}$  must
 be  such  that  the  $A_1$  term  of
Eq.~\ref{method:eq:max_A_k}  is  maximized  subject  to  the  near  disjointness
constraint of Eq.~\ref{eq:MultMeanFieldConstraintEpsilon}, which becomes
\begin{equation}
  \label{method:eq:constraints}
  Q_1 \left(\sum  \limits_{u  = 1  \dots  L} \mathbbm{1}(\bX_{i_u} = v_u) < C \right)\leq \epsilon \;,
\end{equation}
under our cardinality-based clamping  scheme defined by Eq.~\ref{eq:cardSplit2}.
Performing    this   maximization    using    a    standard   Lagrangian    Dual
procedure~\cite{Boyd04} requires evaluating the  constraint and its derivatives.
Despite the potentially exponentially large number of terms involved, we can
  do this in one of two ways. In both cases, the 
   Lagrangian    Dual procedure reduces to  a series of  unconstrained Mean Field  minimizations  
  with well known additional potentials.
\begin{enumerate}{\leftmargin=1.25em}
\itemsep0em

\item When $C$ is close to $0$ or  to $L$, the Lagrangian term can be treated as
  a    specific   form    of   pattern-based    higher-order   potentials,    as
  in~\cite{Vineet14,Fleuret08a,Kohli12,Arnab15}.

\item When $C$ is both substantially greater  than zero and smaller than $L$, we
  treat $ \sum_{u = 1  \dots L} \mathbbm{1}(\bX_{i_u} = v_u)$ as a
  large sum of independent  random  variables  under $Q_1$.  We
  therefore use a  Gaussian approximation to replace  the cardinality constraint
  by a simpler  linear one, \comment{on a  sum of variables $q^1_i$.} and finally add 
  unary potentials to the MF problem. Details are provided
  in the supplementary material.
 
\end{enumerate}
We will encounter  the first situation when tracking pedestrians  and the second
when  performing semantic  segmentation, as  will  be discussed  in the  results
section. 

\comment{ segmentation, instead of enforcing that at least $C$ pixels take label $v$, we
  enforce that the expected number of pixels taking label $v$ is large enough.}

\subsubsection{Handling an Arbitrary Number of Nodes}

Recall from  Section~\ref{sec:partioningStateSpace} that,  in the  general case,
there can  be an  arbitrary number of  modes. They correspond  to the  leaves of
a binary tree created by a succession of cardinality-based splits.
Let us therefore  consider mode $k$ for  $1 \leq k \leq  K$. Let $B$ be  the set of
branching points on the path leading to it. The {\it near} disjointness~\ref{eq:MultMeanFieldConstraintEpsilon},
can be enforced with only $| B |$ constraints.
For each $b \in B$, there is a list
of variables $i^b_1,\dots,i^b_{L^b}$, a  list of values $v^b_1,\dots,v^b_{L^b}$,
a cardinality threshold $C^b$, and a  sign for the inequality $\geq_b$ that
define a constraint
\begin{equation}
  Q_k \left( \sum \limits_{u  = 1  \dots  L^b} \mathbbm{1}(\bX_{i^b_u} = v^b_u)  \geq_b C^b \right) \leq \epsilon 
  \label{method:eq:branchconstraints}
\end{equation}
of  the  same  form  as that  of  Eq.~\ref{method:eq:constraints}.   It  ensures
disjointness with all the modes in the subtree  on the side of $b$ that mode $k$
does not belong to.  Therefore, we can solve the constrained maximization problem
of Eq.~\ref{method:eq:max_A_k}, as in  Section~\ref{sec:twoModes}, but with $| B
|$ constraints  instead of only  one.

\section{Selecting Variables to Clamp}
\label{sec:clamp}

 We  now
present an approach to choosing   the   variables  $i_1,\dots,i_L$   and   the   values
$v_1,\dots,v_L$, which define the cardinality splits of Eqs.~\ref{eq:cardSplit1}
and~\ref{eq:cardSplit2}, that relies on phase transitions in  the graphical model.

To this end, we first introduce a  {\it temperature} parameter in our model that
lets us  smooth the probability  distribution we  want to approximate. This well known
parameter for physicists~\cite{Kadanoff09} was used in a different context in vision by~\cite{Premachandran14}.
 We study its influence on the corresponding MF approximation and how we can exploit
the resulting behavior to select appropriate values for our variables.

\subsection{Temperature and its Influence on Convexity}
\label{sec:tempearature}

We  take the  temperature  $T$  to be  a  number that  we  use  to redefine  the
probability distribution of Eq.~\ref{eq:related:mrf} as
\begin{equation}
  \vspace{-1mm}
  P^T(\bx) = \dfrac{1}{Z^T}e^{-\small{\dfrac{1}{T}}\mE(\bx)} \;,
  \label{method:eq:temp}
\end{equation}
where $Z^T$ is the partition function that normalizes $P^T$ so that its integral
is one.   For $T=1$, $P^T$ reduces  to $P$. As  $T$ goes to infinity,  it always
yields  the same  Maximum-A-Posteriori  value but  becomes increasingly  smooth.
When  performing  the MF  approximation  at  high $T$,  the  first  term of  the
KL-Divergence,  the convex  negative entropy,  dominates and  makes the  problem
convex. As  $T$ decreases, the  second term  of the KL-Divergence,  the expected
energy, becomes dominant, the function stops  being convex, and local minima can
start to appear.  In the supplementary material, we introduce a physics-inspired proof that, in the case of a dense Gaussian CRF~\cite{Kraehenbuehl13}, we can approximate and upper-bound, in  closed-form, the  {\it critical temperature} $T_c$ at which
the  KL divergence  stops being  convex. We validate experimentally this prediction, 
using directly the {\it denseCRF} code from~\cite{Kraehenbuehl13}. This  makes it  easy to
define a temperature  range $[1,T_{max}]$ within which to look  for $T_c$. For a
generic  CRF,  no  such computation  may  be  possible  and  the range  must  be
determined empirically.

\subsection{Entropy-Based Splitting}

We describe here our approach to  splitting $\mX$ into $\mX_1$ and $\mX_2$
at the root node of the tree. The subsequent splits are done in exactly the same
way.  The variables  to be clamped are  those whose value change  from one local
minimum to another so that we can force the exploration of both minima.

To  find them,  we start  at $T_{max}$,  a temperature  high enough  for the  KL
divergence  to be  convex  and  progressively reduce  it.   For each  successive
temperature, we perform the MF approximation  starting with the estimate for the
previous one to speed up the computation.   When looking at the resulting set of
approximations starting from the lowest  temperature ones $T=1$, a telltale sign
of increasing convexity  is that the assignment of some  variables that were very
definite suddenly becomes uncertain. Intuitively, this happens when the CRF terms
that  bind  variables   is  overcome  by  the  entropy   terms  that  encourage
uncertainty. In physical terms, this can be viewed as a local phase-transition~\cite{Kadanoff09}.


Let $T$  be a temperature greater  than $1$ and  let $Q^{T}$ and $Q^{1}$  be the
corresponding  Mean  Field  approximations, with  their  marginal  probabilities
$q^{T}_{i}$  and $q^{1}_{i}$ for each variable $i$.   To  detect such  phase  transitions, we  compute
 \vspace{-1mm}
 \begin{equation}
 \label{eq:entropy}
 \delta_{i}(T) = \mathbbm{1}[\mH(q^{T}_{i})>h_{high}]\mathbbm{1}[\mH(q^{1}_{i})<h_{low}] \;,
 \end{equation}
 \vspace{-1mm}
 for all $i$, where $\mH$ denotes  the individual entropy. \comment{The values of the
 thresholds  $h_{high}$  and  $h_{low}$  can  vary  slightly  depending  on  the
 application, the family of temperatures, and the number of labels.}
 
All  variables  and labels  with  positive  $\delta_{i}$ become  candidates  for
clamping. If there are none, we increase the temperature. If there
are several, we  can either pick one  at random or use domain  knowledge to pick
the most suitable subset and values as will be discussed in the Results Section.



\section{Results}
\label{sec:results}

We  first  use  synthetic  data  to demonstrate  that  MMMF  can  approximate  a
multi-modal probability  density function better  than both standard MF  and the
recent approach  of~\cite{Weller15}, which  also relies  on clamping  to explore
multiple  modes.   We  then  demonstrate  that  this  translates  to  an  actual
performance   gain    for   two    real-world   algorithms---one    for   people
detection~\cite{Fleuret08a}          and          the         other          for
segmentation~\cite{Chen15b,Yu16b}---both  relying on  a  traditional Mean  Field
approach.  We will make all our code and test datasets publicly available.

The parameters that control MMMF are the  number of modes we use, the cardinality threshold $C$
at   each   split,   the   $\epsilon$   value   of
Eq.~\ref{eq:MultMeanFieldConstraintEpsilon},  the  entropy thresholds  $h_{low}$
and $h_{high}$ of Eq.~\ref{eq:entropy},  and the temperature $T_{\max}$ introduced
in   Section~\ref{sec:clamp}.    In   all    our   experiments,   we   use
$\epsilon=10^{-4}$, $h_{low} = 0.3$, and  $h_{high} = 0.7$.  As discussed in
  Section~\ref{sec:clamp}, when  the CRF is  a dense Gaussian CRF,  we can
  approximate and upper bound the critical  temperature $T_c$ in closed-form and  we simply take
  $T_{\max}$ to  be this upper bound to guarantee that $T_{\max}>T_c$.   Otherwise, we
  choose  $T_{\max}$ empirically  on a  small validation-set  and fix  it during
  testing.


\subsection{Synthetic Data}
\label{sec:synthetic}

To  demonstrate  that our  approach  minimizes the  KL-Divergence better  than  both
standard  MF  and   the  clamping  one  of~\cite{Weller15},  we   use  the  same
experimental protocol to generate conditional  random fields with random weights
as  in~\cite{Eaton09,Weller14,Weller15}.  Our  task  is then  to  find the  MMMF
approximation with  lowest KL-Divergence  for any given  number of  nodes.  When
that number is one, it reduces  to MF.  Note that the authors of~\cite{Weller15}
look for an  approximation of the log-partition function, which  is strictly the
same  as minimizing  the  KL-Divergence, as  demonstrated  in the  supplementary
material. Because  it involves randomly  chosen positive and  negative weights,
this problem effectively mimics difficult  real-world ones with repulsive terms,
uncontrolled loops, and strong correlations.

In Fig.~\ref{fig:evaluation:KL}, we plot the  KL-Divergence as a function of the
number of modes used to approximate the distribution on the standard benchmarks.
These modes are obtained using either  our entropy-based criterion as described in Section~\ref{sec:clamp}, 
or the MaxW one of~\cite{Weller15},
which  we will  refer to  as \BASEM{}.   It involves  sequentially clamping  the
variable having the largest  sum of absolute values of  pairwise potentials  for edges
linking it to its neighbors. It  was shown to be one of the best
methods   among  several   others,  which   all  performed   roughly  similarly.
\comment{We  don't  use additional  techniques  on  top  of MaxW,  as  suggested
  by~\cite{Weller15}, such as the {\it stripping to the core} or {\it singleton}
  ones,  since that  these  methods  involve coupling  MF  with other  inference
  methods and are hardly applicable to our problems.}  \comment{Note that, since
  that the value of the log-partition function $\log(Z)$ is unknown, we can only
  compute the KL-Divergence up to a constant.  However, to estimate how close to
  zero  we are  and  to  get a  better  sense of  the  improvement,  we use  the
  upper-bound of $\log(Z)$ provided by a TRW approximation~\cite{Wainwright02}.}
\comment{
\begin{itemize}
  
\item  {\bf  Attractive  grid.}  $N  \times N$  grid  with  attractive  pairwise
  potentials drawn uniformly from $[0,6]$ and unaries from $[-2,2]$.
  
\item {\bf Mixed grid.} $N \times N$ grid with attractive and repulsive pairwise
  potentials drawn uniformly from $[-6,6]$ and unaries from $[-2,2]$.
  
\item {\bf Attractive  random.}  49-node graph with the same  number of randomly
  chosen edges  as the $N \times  N$ grid, attractive pairwise  potentials drawn
  uniformly from $[0,6]$, and unaries drawn from $[-2,2]$.
  
\item {\bf  Mixed random.}  49-node graph  with the same  number of  of randomly
  chosen  edges as  the $N  \times N$  grid, attractive  and repulsive  pairwise
  potentials drawn uniformly from $[-6,6]$, and unaries drawn from $[-2,2]$.
\end{itemize}
}
In   our    experiments,   we    used   the   phase-transition    criterion   of
Section~\ref{sec:clamp} to select  candidate variables to clamp.  We then either
randomly chose the  group of $L$ variables  to clamp or used  the MaxW criterion
of~\cite{Weller15} to select the best $L$  variables. We will refer to the first
as \OURSR{} and to the second as \OURSM{}. Finally, in all cases, $C=L$ and the
values $v_u$ correspond to the ones taken by the MAP of the mode split. 

In  Fig.~\ref{fig:evaluation:KL}, we  plot the  resulting curves  for $L=1$  and
$L=3$, evaluated on 100 instances. \OURSR{} performs better than the 
method~\BASEM{}  in most cases, even though it
does  not use  any knowledge  of the  CRF internals,  and \OURSM{},  which does,
performs even better. 
The results on the $13\times13$ grid demonstrate the advantage
of clamping variables by groups when the CRF gets larger.

\comment{
\begin{figure}[ht!]
\begin{center}
\includegraphics[width=0.50\textwidth]{{KL5}.png} \\
\end{center}
\caption{Lower  bounds of  the KL-divergence  for our  clamping method  and that
  of~\cite{Weller15}.}
\label{fig:evaluation:KL}
\end{figure}
}

\begin{figure*}[ht!]
\vspace{-2mm}
\begin{center}
\begin{tabular}{c@{}c@{}c@{}c}
\setlength{\tabcolsep}{0pt}
\includegraphics[trim=8mm 8mm 8mm 8mm, clip, width=0.25\textwidth]{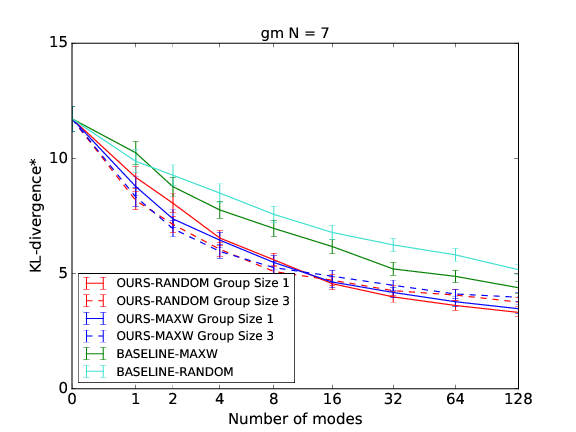} & 
\includegraphics[trim=8mm 8mm 8mm 8mm, clip, width=0.25\textwidth]{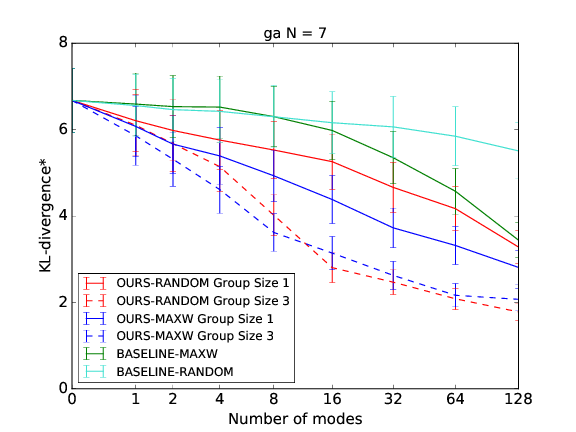} &
\includegraphics[trim=8mm 8mm 8mm 8mm, clip, width=0.25\textwidth]{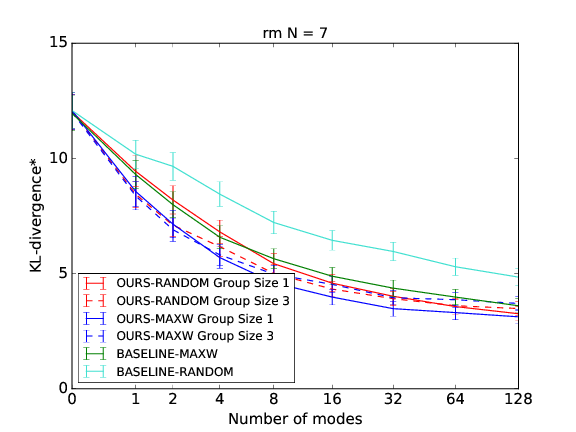} &
\includegraphics[trim=8mm 8mm 8mm 8mm, clip, width=0.25\textwidth]{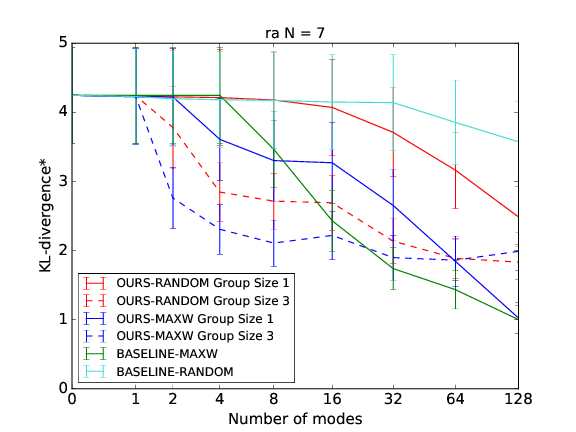} \\
\includegraphics[trim=8mm 8mm 8mm 8mm, clip, width=0.25\textwidth]{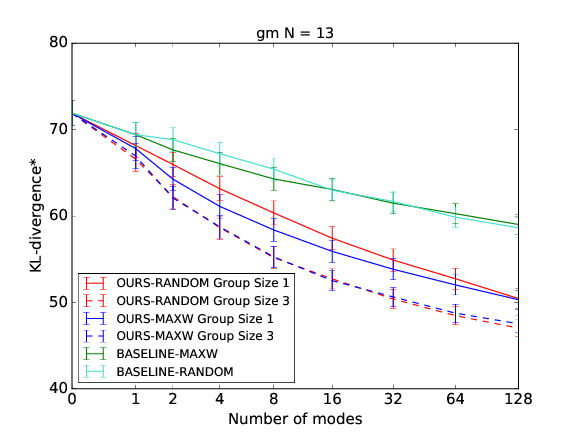} & 
\includegraphics[trim=8mm 8mm 8mm 8mm, clip, width=0.25\textwidth]{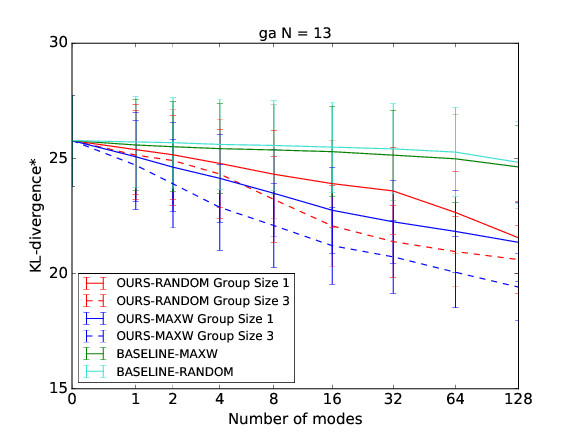} &
\includegraphics[trim=8mm 8mm 8mm 8mm, clip, width=0.25\textwidth]{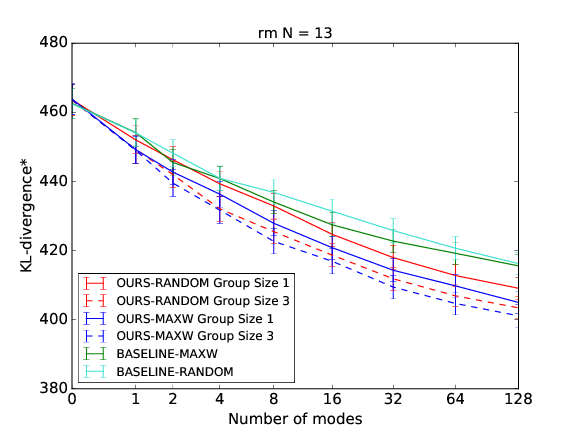} &
\includegraphics[trim=8mm 8mm 8mm 8mm, clip, width=0.25\textwidth]{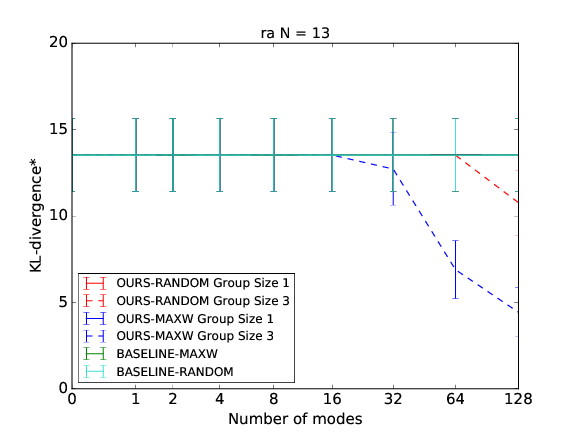} \\
 \sttt{Mixed grid} &  \sttt{ Attractive grid} &
 \sttt{Mixed random} &  \sttt{ Attractive random}  \\
\end{tabular}
\end{center}
\vspace{-0.2cm}
   \caption{KL-divergence   using   either   our   clamping   method   or   that
     of~\cite{Weller15} averaged  over 100  trials. The vertical  bars represent
     standard deviations.  {\bf Attractive} means  that pairwise terms are drawn
     uniformly   from  $[0,6]$   whereas  {\bf   Repulsive}  means   drawn  from
     $[-6,6]$. {\bf  Grid} indicates a grid  topology for the CRF,  whereas {\bf
       Random}  indicates that  the connections  are chosen  randomly such  that
     there are as many  as in the grids. We ran our  experiments with both
     $7 \times 7$ and $13 \times 13$ variables CRFs.}
\label{fig:evaluation:KL}
\vspace{-0.3cm}
\end{figure*}

\subsection{Multi-modal Probabilistic Occupancy Maps}
\label{sec:MMPOM}
The  Probabilistic  Occupancy  Map   (POM)  method~\cite{Fleuret08a}  relies  on
Mean Field inference for pedestrian detection.  More specifically, given several
cameras  with overlapping  fields of  view of  a discretized  ground plane,  the
algorithm  first  performs  background   subtraction.   It  then  estimates  the
probabilities of  occupancy at  every discrete  location as  the marginals  of a
product law minimizing the KL divergence from the ``true'' conditional posterior
distribution, formulated  as in  Eq.~\ref{eq:related:mrf} by defining  an energy
function.  Its  value is  computed by  using a  generative model:  It represents
humans  as simple  cylinders projecting  to  rectangles in  the various  images.
Given the  probability of presence or  absence of people at  different locations
and known camera  models, this produces synthetic images whose  proximity to the
corresponding background subtraction  images is measured and used  to define the
energy.

This algorithm is usually very robust but can fail when multiple interpretations
of a  background subtraction image  are possible.   This stems from  the limited
modeling  power  of  the  standard  MF  approximation,  as  illustrated  in  the
supplementary material.  We show here that,  in such cases, replacing MF by MMMF
while retaining the rest of the framework yields multiple interpretations, among
which the correct one is usually to be found.


Fig.~\ref{fig:evaluation:pomPlot}  depicts what  happens when  we replace  MF by
MMMF  to approximate  the true  posterior, while  changing nothing  else to  the
algorithm.    To   generate    new   branches    of   the    binary   tree    of
Section~\ref{sec:partioningStateSpace}, we find potential  variables to clamp as
described in  Section~\ref{sec:clamp}. Among those,  we clamp the one  with the
largest  entropy gap---$\mH(q^{T}_{i})-\mH(q^{1}_{i})$,  using the  notations of
Eq.~\ref{eq:entropy}---and  its  neighbors on  the  grid.   When evaluating  our
cardinality constraint, we take $C$ to be 1, meaning that one branch of the tree
corresponds to no one in the neighborhood of the selected location and the other
to at least  one person being present in this  neighborhood.  Since we typically
create those  locations by discretizing the  ground plane into $10  cm \times 10
cm$ grid cells, this forces the two newly instantiated modes to be significantly
different  as  opposed  to  featuring  the  same  detection  shifted  by  a  few
centimeters.   In  Fig.~\ref{fig:evaluation:pomPlot},  we plot  the  results  as
dotted  curves  representing  the  MODA  scores as  functions  of  the  distance
threshold  used to  compute them~\cite{Bernardin08}.   In all  cases, we  used 4
modes    for    the   MMMF    approximation    and    followed   the    DivMBest
evaluation metric~\cite{Batra13} to  produce a score  by selecting among the  4 detection
maps corresponding  to each mode the  one yielding the highest  MODA score. This
produces red  dotted MMMF curves that  are systematically above the  blue dotted
MF.

However, to turn this  improvement into a practical technique, we  need a way to
choose among the  4 possible interpretations without using the  ground truth. \pb{We
  use temporal  consistency to jointly find the best sequence of modes, and reconstruct
  trajectories from this sequence.}  In  the original  algorithm, the  POMs
computed at  successive instants  were used  to produce  consistent trajectories
using  the a  K-Shortest  Path (KSP)  algorithm~\cite{Berclaz11}. This  involves
building a graph in which each ground  location at each time step corresponds to
a node and  neighboring locations at consecutive time steps  are connected.  KSP
then finds a set of node-disjoint shortest paths in this graph where the cost of
going through a location is proportional  to the negative log-probability of the
location in  the POM~\cite{Suurballe74}.  Since MMMF  produces multiple POMs,
  we then  solve a multiple  shortest-path problem in  this new graph,  with the
  additional constraint that at each time step  all the paths have to go through
  copies  of the  nodes corresponding  to the  same mode,  as described  in more
  details in the supplementary material.

\comment{Since MMMF produces  multiple POMs, one
for each  mode, at each  time-step, we duplicate the  KSP graph nodes,  once for
each mode.   Each node is then  connected to each copy  of neighboring locations
from previous and  following time steps. We then solve  a multiple shortest-path
problem in this new graph, with the additional constraint that at each time step
all the paths have  to go through copies of the nodes  corresponding to the same
mode.   This larger  problem is  NP-Hard and  cannot be  solved by  a polynomial
algorithm such as  KSP.  We therefore use a Mixed-Integer  Linear solver after a
KSP-based pruning procedure,  as explained in more details  in the supplementary
material.}

The solid blue lines in Fig.~\ref{fig:evaluation:pomPlot} depict the MODA scores
when using  KSP and  the red  ones the  multi-modal version,  which we  label as
KSP$^*$. The MMMF curves are again above  the MF ones.  This makes sense because
ambiguous situations rarely persist  for more than a few a  frames. As a result,
enforcing temporal consistency eliminates them. 


%
\begin{figure}[ht!]
\begin{center}
\begin{tabular}{@{}cc}
\includegraphics[width=0.45\textwidth]{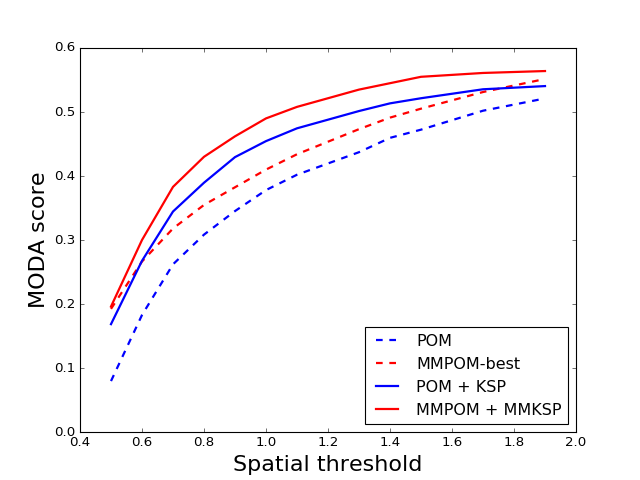} & \includegraphics[width=0.45\textwidth]{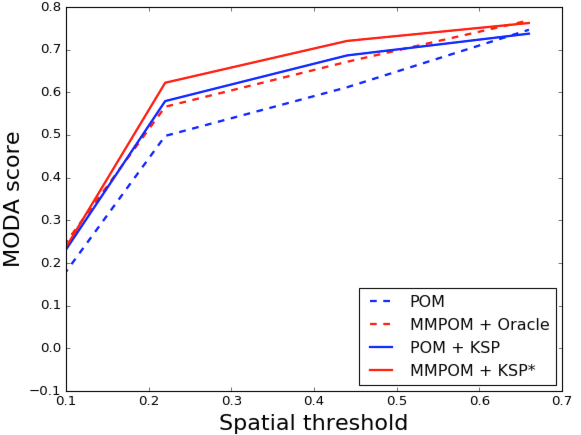}  \\
\includegraphics[width=0.45\textwidth]{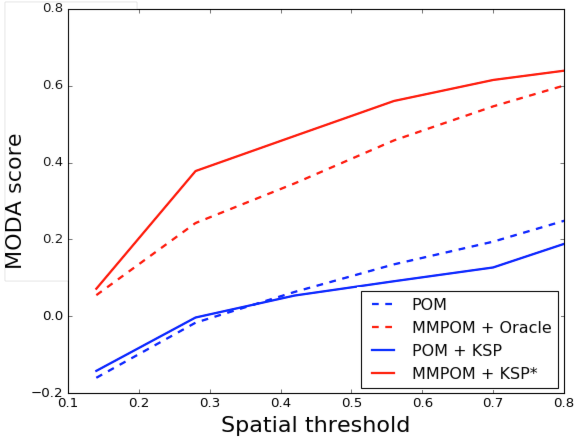} & \includegraphics[width=0.45\textwidth]{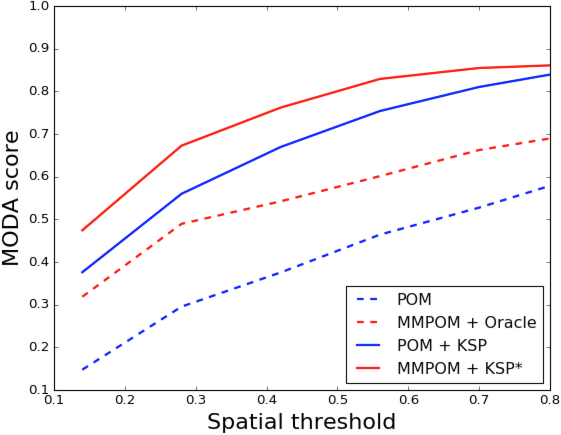} \\
\end{tabular}
\end{center}
\vspace{-0.2cm}
   \caption{Replacing  MF by  MMMF in  the POM  algorithm~\cite{Fleuret08a}. The
     blue curves are  MODA scores~\cite{Bernardin08} obtained using  MF and the
     red ones  scores using  MMMF. They  are shown as  solid lines  when temporal
     consistency was enforced  and as dotted lines otherwise. Note  that the red
     MMMF lines  are above corresponding  blue MF ones  in all cases.   (a) 1000
     frames from the  MVL5~\cite{Mandeljc12} dataset using  a single camera.
     (b) 400 frames from the Terrace dataset~\cite{Berclaz11} using two cameras.
     (c) 80 frames  of the EPFL-Lab dataset~\cite{Berclaz11} using  a single camera.
     (d) 80 frames from the EPFL-Lab dataset~\cite{Berclaz11} using two cameras. }
\label{fig:evaluation:pomPlot}
\vspace{-0.3cm}
\end{figure}
\subsection{Multi-Modal Semantic Segmentation}
\label{sec:semanticSegmentation}

CRF-based semantic segmentation is one of best known application of MF inference
in   Computer    Vision   and   many    recent   algorithms   rely    on   dense
CRF's~\cite{Kraehenbuehl13} for this purpose.  We demonstrate here that our MMMF
approximation  can   enhance  the  inference   component  of  two   such  recent
algorithms~\cite{Chen15b,Yu16b}  on the Pascal VOC 2012 segmentation dataset and
the MPI video segmentation one ~\cite{Galasso13}.

\vspace{-4mm}
\paragraph{Individual VOC Images}
  
We write the  posterior in terms of the CRF of~\cite{Chen15b},  which we try to
  approximate.   To   create    a   branch    of   the    binary   tree    of
Section~\ref{sec:partioningStateSpace}, we first find the potential variables to
clamp  as  described  in Section~\ref{sec:clamp}.   As  in~\ref{sec:MMPOM},  we
select  the  ones   in  the  sliding  window  with  the   largest  entropy  gap,
$\mH(q^{T}_{i})-\mH(q^{1}_{i})$.  We then  take $C$ to be  $L/2$ when evaluating
our cardinality  constraint, meaning that we  seek the dominant label  among the
selected variables and split the state space into those for which more than half
these variables take this value and those in which less than half do.

\begin{figure}[ht!]
\begin{center}
\begin{tabular}{@{}cc}
\includegraphics[width=0.42\textwidth]{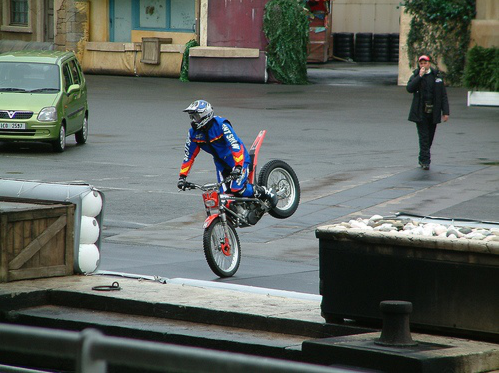} &
\includegraphics[width=0.42\textwidth]{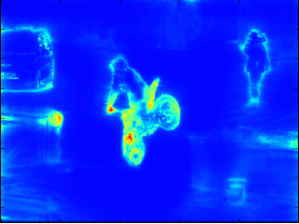}  \\
(a) & (b) \\
\includegraphics[width=0.42\textwidth]{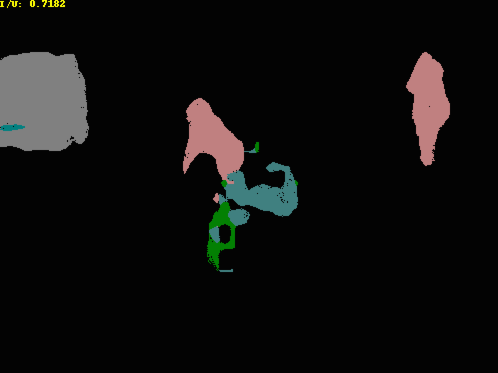} &
\includegraphics[width=0.42\textwidth]{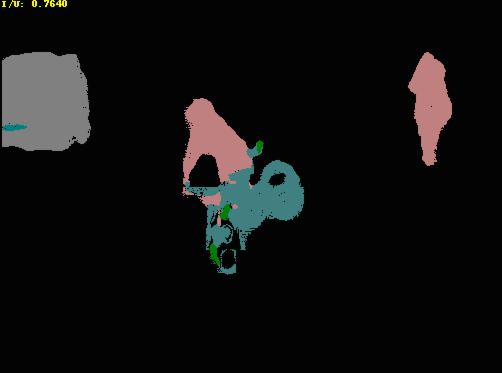}  \\
(c) & (d) \\
\end{tabular}
\end{center}
\vspace{-0.2cm}
   \caption{Qualitative semantic segmentation.  (a) Original image.  (b) Entropy
     gap.  (c)   Labels  with   maximum  a   Posteriori  Probability   after  MF
     approximation. (d)  Labels with  maximum a  Posteriori Probability  for the
     best mode of the MMMF approximation.}
\label{fig:evaluation:VOC_qualitative}
\vspace{-0.3cm}
\end{figure}
Fig.~\ref{fig:evaluation:VOC_qualitative} illustrates the results  on an image of
the  VOC dataset. To evaluate  such results  quantitatively, we  first use  the
DivMBest metric~\cite{Batra13}, as we  did in Section~\ref{sec:MMPOM}. We assume
we have an oracle that can select the best mode of our multi-modal approximation
by looking  at the  ground truth.  Fig.~\ref{fig:evaluation:VOC_quantitative_1}  depicts the
results on  the validation set of  the VOC 2012  Pascal dataset in terms  of the
average  intersection over  union (IU)  score  as a  function of  the number  of
modes. When only 1 mode is used,  the result boils down to standard MF inference
as  in~\cite{Chen15b}.   Using 32  yields  a  $2.5\%$  improvement over  the  MF
approximation. This may seem small until one considers that we {\it only} modify
the  algorithm's  inference   engine  and  leave  the   unary  terms  unchanged.
In~\cite{Chen15b,Zheng15},   this   engine   has  been   shown   to   contribute
approximately  $3\%$ to  the overall  performance,  which means  that we  almost
double its effectiveness. For analysis purposes, we implemented two baselines:
\begin{itemize}{\leftmargin=1.25em}
\itemsep0em

  \item Instead of clamping groups of variables, we only clamp the variable with
    the maximum  entropy gap  at each  step.  As  depicted by  the red  curve in
    Fig.~\ref{fig:evaluation:VOC_quantitative_1},   this  has   absolutely  no   effect  and
    illustrates the importance of clamping  groups of variable instead of single
    ones as in~\cite{Weller15}.

  \item The DivMBest approach~\cite{Batra13} first computes a MAP and then adds a
    penalty term  to the energy function  to find another MAP  that is different
    from the first. It then repeats the process. We adapted this approach for MF
    inference. The  green curve in Fig.~\ref{fig:evaluation:VOC_quantitative_1}  depicts the
    result, which MMMF outperforms by $1.5\%$.
\end{itemize}
\begin{figure}[ht!]
\begin{center}
\includegraphics[width=0.9\textwidth]{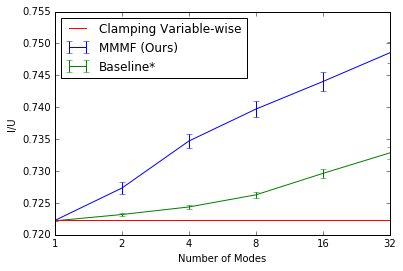} 
IU score for the best mode\\
\end{center}
   \caption{Quantitative semantic segmentation on VOC 2012. IU score for best mode
  as a  function of  the number  of modes. MMMF  in blue,  baselines in  red and
  green.}
\label{fig:evaluation:VOC_quantitative_1}
\end{figure}
\paragraph{Semantic Video Segmentation.}

We  ran  the  same experiment  on  the  images  of  the MPI  video  segmentation
dataset~\cite{Galasso13} using  the CRF of~\cite{Yu16b}.   In this case,  we can
exploit temporal consistency  to avoid having to use an  oracle and nevertheless
get an exploitable result, as we did in Section~\ref{sec:MMPOM}. Furthermore, we
can do this in spite of the relatively low frame-rate of about 1Hz. 

More specifically, we  first define a compatibility  measure between consecutive
modes  based on  label probabilities  of matching  key-points, which  we compute
using a key-point matching algorithm~\cite{Revaud16}. We then compute a shortest
path  over  the   sequence  of  modes,  taking  into   account  individual  mode
probabilities given  by Eq.~\ref{method:eq:mk}.   Finally, we  use only  the MAP
corresponding to the  mode chosen by the shortest path  algorithm to produce the
segmentation. In  Fig.~\ref{fig:evaluation:VOC_quantitative_2}, we again report  the results
in  terms of  IU  score. This  time  the improvement  is  around $2.4\%$,  which
indicates  that imposing  temporal consistency  very substantially  improves the
quality of the inference.
To the best of our knowledge, other state of the art video semantic segmentation
methods  are  not applicable  for  such  image sequences.~\cite{Hur16}  requires
non-moving scenes and a super-pixel  decomposition, which prevents using all the
dense  CRF-based image  segmentors.~\cite{Kundu16}  was only  applied to  street
scenes  and requires  a  much higher  frame  rate to  provide  an accurate  flow
estimation.  
\begin{table}[]
\centering
\label{table:evaluation:IOU}
\begin{tabular}{|l|l|}
\hline
\textbf{Method} & \textbf{Mean IOU} \\ \hline
MF              & 44.9\%            \\ \hline
\cite{Weller15} + Temp & 44.9\%            \\ \hline
MMMF + Temporal & 47.3\%            \\ \hline
MMMF-Best   & 53.2\%              \\ \hline
\end{tabular}
\caption{Quantitative semantic segmentation MPI dataset~\cite{Galasso13}.}
\label{fig:evaluation:VOC_quantitative_2}
\end{table}


\section{Conclusion}

We have shown  that our MMMF aproach  makes it possible to add  structure to the
standard MF approximation of CRFs and  to increase the performance of algorithms
that  depend   on  it.   In   effect,  our  algorithm  creates several  alternative  MF
approximations  with probabilities  assigned to  them, which  effectively models
complex situations in which more than one interpretation is possible.

Since MF  has recently  been integrated  into structured  learning architectures
through the  Back Mean-Field procedure~\cite{Domke13,Li14,Zheng15,Arnab15}, future  work will
aim to replace MF by MMMF in this context as well.

{\small
  \bibliographystyle{ieee}
  \bibliography{string,vision,learning,optim}
}

\cleardoublepage
\begin{appendices}

\section{Proofs for Multi-Modal Mean-Fields via Cardinality-Based Clamping}

This document provides technical details and proofs related to Section 5. We first prove the approximation of the KL-Divergence used in Eq. 9. Then, we show that the problem that we are trying to solve in Eq. 9, the minimization of the KL-Divergence, is actually equivalent to the one solved by~\cite{Weller15}, namely, finding an approximation to the log-partition function. It eventually justifies the benchmark experiments ran in 7.1. Finally, we justify the Gaussian approximation used in the case of large clamping groups in 5.2.1-(2).

\subsection{Minimising the KL-Divergence}
Let us see how the KL-Divergence between $Q_{MM}$ and P of Eq. 3 can be minimised with respect to the parameters $m_k$ and to the distributions $Q_k$, leading to Eq. 9. We reformulate the minimisation problem up to a constant approximation factor of order $\epsilon \log(\epsilon)$.

First, remember that our minimisation problem enforces the near-disjointness condition,
\begin{equation}
\forall k \neq k^{\prime} \sum_{\bx \in \mX_k^{\prime}} Q_{k}(\bx) \leq \epsilon
\;,
\label{eq:MultMeanFieldConstraintEpsilon}
\end{equation}
 between the elements of the mixture. 
 
Let us then prove the following useful Lemma. 
\begin{lemma}
~\label{KL:lemma:epsilon}
For all mixture element $k \leq K$,
\begin{equation}
 \sum \limits_{\bx \in \mX}  Q_{k}(\bx) 
  \log \left( \sum \limits_{k^\prime \leq K}  m_{k^\prime} Q_{k^\prime}(\bx) \right) =  \sum \limits_{\bx \in \mX}  Q_{k}(\bx) 
  \log \left( m_k Q_{k}(\bx) \right) + \mO(\epsilon \log \epsilon)\;.
\end{equation}
\end{lemma}

\begin{proof}
Let $k$ be the index of a mixture component $k \leq K$, and let us denote the approximation error
\begin{equation}
 \delta_k = \sum \limits_{\bx \in \mX}  Q_{k}(\bx) \log \left( \sum \limits_{k^\prime \leq K}  m_{k^\prime} Q_{k^\prime}(\bx) \right) - \sum \limits_{\bx \in \mX} Q_{k}(\bx) \log \left( m_k Q_{k}(\bx) \right)  \;.
\end{equation}
Then, we use the near-disjointness condition to bound $\delta_k$,
\begin{align}
\delta_k & \leq \underbrace{\sum \limits_{\bx \in \mX_k} Q_{k}(\bx) \log \left( 1 + \dfrac{ \sum \limits_{k^\prime \neq k}  m_{k^\prime} Q_{k^\prime}(\bx)}{Q_{k}(\bx)} \right)}_{I} +  \underbrace{\sum \limits_{\bx \in \mX \setminus \mX_k} Q_{k}(\bx) \log \left( 1 + \dfrac{ \sum \limits_{k^\prime \neq k}  m_{k^\prime} Q_{k^\prime}(\bx)}{Q_{k}(\bx)} \right)}_J
\end{align}
We first use the well known inequality $\log(1 + x) \leq x$ in order to upper bound $I$,
\begin{align}
I &\leq \sum \limits_{\bx \in \mX_k}  Q_{k}(\bx)  \dfrac{ \sum \limits_{k^\prime \neq k}  m_{k^\prime} Q_{k^\prime}(\bx)}{Q_{k}(\bx)} \\
&\leq \sum \limits_{k^\prime \neq k}  \sum \limits_{\bx \in \mX_k}  m_{k^\prime} Q_{k^\prime}(\bx) \\
&\leq \sum \limits_{k^\prime \neq k} \epsilon \\
&\leq \mO (\epsilon) \;.
\end{align}
The second term, $J$, can then be upper-bounded using the fact that the $m_{k^\prime}$ and $Q_{k^\prime}$ are mixture weights and probabilities and hence $\sum \limits_{k^\prime \neq k}  m_{k^\prime} Q_{k^\prime}(\bx) \leq 1$ for all $\bx$. Therefore,
\begin{align}
J &\leq \sum \limits_{\bx \in \mX \setminus \mX_k} Q_{k}(\bx) \log \left( 1 + \dfrac{1}{Q_{k}(\bx)} \right) \\
&\leq \sum \limits_{\bx \in \mX \setminus \mX_k} - Q_{k}(\bx) \log \left( Q_{k}(\bx) \right) \\
&\leq \sum \limits_{k^\prime \neq k} \sum \limits_{\bx \in \mX_{k^\prime}}  - Q_{k}(\bx) \log \left( Q_{k}(\bx) \right) \;.
\end{align}
Furthermore, for all $k^\prime \neq k$, the near-disjointness condition enforces that $\sum \limits_{\bx \in \mX_{k^\prime}}  Q_{k}(\bx) \leq \epsilon$. Under this constraint, on each of the subsets $\mX_{k^\prime}$, the maximal entropy is reached if $Q_{k}(\bx) = \dfrac{\epsilon}{\mid \mX_{k^\prime} \mid}$ for all $\bx$ in $\mX_{k^\prime} $. And, therefore
\begin{align}
\sum \limits_{\bx \in \mX_{k^\prime}}  - Q_{k}(\bx) \log \left( Q_{k}(\bx) \right) &\leq \epsilon \log \left( \dfrac{\mid \mX_k^\prime \mid}{\epsilon} \right)\\
&\leq \mO(\epsilon \log \epsilon) + \mO(\epsilon) \;,
\end{align}
where the factor $\log(\mid \mX_k \mid)$, which is of the order of the number of variables, has been integrated in the constant.

Hence,
\begin{align}
J &\leq \sum \limits_{k^\prime \neq k} \sum \limits_{\bx \in \mX_{k^\prime}}  - Q_{k}(\bx) \log \left( Q_{k}(\bx) \right)\\
& \leq \mO(\epsilon \log \epsilon) + \mO(\epsilon) \;,\\
\end{align}
which terminates the proof.

\end{proof}

We can then move on to the minimisation of the KL-Divergence
\begin{align}
\min \limits_{\hat{m},\hat{q}}\KL(Q_{MM} \| P) &= \min \limits_{\hat{m},\hat{q}} \sum \limits_{\bx \in
  \mX} \sum \limits_{k \leq K} Q_{MM}(\bx) \log \left( \dfrac{Q_{MM}(\bx)}{P(\bx)} \right) \\
 &= \min \limits_{\hat{m},\hat{q}} \sum \limits_{\bx \in
  \mX} \sum \limits_{k \leq K} Q_{MM}(\bx) \log \left( \dfrac{Q_{MM}(\bx)}{e^{-E(\bx)}} \right) + \log(Z) \\
  &=\min \limits_{\hat{m},\hat{q}} \sum \limits_{k \leq K} \sum \limits_{\bx \in \mX}  m_k Q_{k}(\bx) 
  \log \left( \dfrac{\sum \limits_{k^\prime \leq K}  m_{k^\prime} Q_{k^\prime}(\bx)}{e^{-E(\bx)}} \right) + \log(Z) \\ 
    &=\min \limits_{\hat{m},\hat{q}} \sum \limits_{k \leq K} \sum \limits_{\bx \in \mX}  m_k Q_{k}(\bx) 
  \log \left( \dfrac{ m_k Q_{k}(\bx)}{e^{-E(\bx)}} \right) + \log(Z) + \mO(\epsilon \log \epsilon) \label{KL:eq:epsilon}\\
   &=\min \limits_{\hat{m}} \left[ \sum \limits_{k \leq K} m_k \log m_k + \sum \limits_{k \leq K} \min \limits_{q_k} \sum \limits_{\bx \in \mX}  m_k Q_{k}(\bx)  \log \left( \dfrac{ Q_{k}(\bx)}{e^{-E(\bx)}} \right) \right] + \log(Z) + \mO(\epsilon \log \epsilon) \\
      &=  \min \limits_{\hat{m}}  \sum \limits_{k \leq K} m_k\log(m_k) - \sum \limits_{k \leq K} m_k A_k + \log(Z) + \mO(\epsilon \log \epsilon) \; , \label{eq:KL_end}
\end{align}
where,
\begin{align*}
A_k & =  \max  \limits_{q^k_i,{i = 1  \dots N}}
  \sum      \limits_{\bx      \in      \mX}     Q_{k}(\bx)      \log      \left(
  \dfrac{e^{-E(\bx)}}{Q_{k}(\bx)}\right) \;.
  \end{align*}

Equation~\ref{KL:eq:epsilon} is obtained using Lemma~\ref{KL:lemma:epsilon}.

Assuming that we are able to compute $A_k$, for all $k$, the minimisation of this KL-Divergence with respect to parameters $m_k$, under the nomalisation constraint
\begin{equation}
\label{eq:mk}
\sum \limits_{k \leq K} m_k = 1 \;,
\end{equation}
is then straightforward and leads to

\begin{equation}
m_k = \dfrac{e^{A_k}}{\sum \limits_{k^\prime \leq K} e^{A_{k^\prime}}} \;.
\end{equation}

\subsection{Equivalence with the approximation of the partition function.}
The recent work of~\cite{Weller15}, that we use as a baseline, looks for the best heuristic to choose the clamping variables. They measure the quality of the approximation through the closeness of the estimated partition function, which they compute as the sum of MF approximated partition functions for each component of the mixture, to the true one. We will now see that this problem is strictly equivalent to the minimisation of the KL-Divergence of Eq. 9.

Indeed, replacing~\ref{eq:mk} in~\ref{eq:KL_end}, we directly obtain that
\begin{align}
\KL(Q_{MM} \| P)  &= \log(\sum \limits_{k^\prime \leq K} e^{A_{k^\prime}}) + \log(Z) + \mO(\epsilon \log \epsilon) \\
& =\log(Z) - \log(\tilde{Z}) + \mO(\epsilon \log \epsilon) \;,
\end{align}
where,
\begin{equation}
\tilde{Z} = \sum \limits_{k^\prime \leq K} e^{A_{k^\prime}} \;,
\end{equation}
is precisely the approximation of the partition function $Z$ proposed by~\cite{Weller15}. In other terms, it is the sum of local variational lower-bounds on clamped subsets of the state space.

\subsection{Gaussian approximation to the cardinality constraint.}
In the following, we explain the Gaussian approximation of the cardinality constraint used in 5.2.1-(2) and in our application to Semantic Segmentation.
Let us consider the case where  we generate
only  two modes  modelled  by $Q_1(\bx)=\prod  q^1_{i}(x_i)$ and  $Q_2(\bx)=\prod
q^2_{i}(x_i)$ and  we seek to  estimate the $q^1_{i}$  probabilities.
The $q^2_{i}$ probabilities are evaluated similarly.

Recall that, each $A_k$ is obtained through the constrained MF optimisation problem
\begin{equation}
\begin{aligned}
\label{eq:constrainedOptim}
&  \max  \limits_{q^k_i,{i = 1  \dots N}}
& &   \sum      \limits_{\bx      \in      \mX}     Q_{k}(\bx)      \log      \left( \dfrac{e^{-E(\bx)}}{Q_{k}(\bx)}\right)\\
& \text{s.t.} & &  Q_1 \left(\sum  \limits_{u  = 1  \dots  L} \mathbbm{1}(\bX_{i_u} = v_u) < C \right)\leq \epsilon \;.
\end{aligned}
\end{equation}

Under the probability $Q_1$, $\sum  \limits_{u  = 1  \dots  L} \mathbbm{1}(\bX_{i_u} = v_u)$ is a sum of independent binary random variables that are non identically distributed, in other words, a Poisson Binomial Distribution. In the general case, there is no closed-form formula for computing the Cumulative Distribution Function of such a distribution from the individual marginals parametrising $Q_1$. However, when $L$ is large ($\geq 10$), the Gaussian approximation is good enough.

Therefore, we use a Gaussian approximation to replace the cardinality constraint by
\begin{equation}
\sum \limits_{u \in \{1 \dots L\}} q^1_{i_u}(v_u) < C + \sigma F^{-1}(1-\epsilon) \;,
\end{equation}
where $F$ is the Gaussian cumulative distribution fonction and $\sigma^2$ the variance, which, in theory should be
\begin{equation}
\sigma^2 = \sum \limits_{u \in \{1 \dots L\}} q^1_{i_u}(v_u)(1-q^1_{i_u}(v_u))\;,
\end{equation}
but which can be either upper-bounded by $\dfrac{L}{4}$ or re-estimated at the beginning of each Lagrangian iteration.

In short, we replace the untractable higher order constraint~\ref{eq:constrainedOptim}, by a simple one involving only  the sum of the MF parameters $q^1_{i_u}(v_u)$.

\section{Computing the Critical Temperature for the Dense Gaussian CRFs}

We first compute analytically the phase transition temperature parameter $T_c$ of 6.2 where the KL-Divergence stops being convex. In the first part {\it Analytical Derivation}, we make strong assumptions in order to be able to obtain a closed form estimation of $T_c$. We then explain how this result helps understanding real cases. In the second part {\it Experimental Analysis}, in order to justify our assumptions, we run experiments under three regimes, one where our assumptions are strictly verified, one which corresponds to a real-life scenario and an intermediate one. This set of experiments shows that our strong assumptions provide a valuable insight for practical applications.

\subsection{Analytical derivation}

Let us take probability distribution $P$ to be defined by a dense Gaussian CRF~\cite{Kraehenbuehl13}. In order to make computation tractable, we assume that the RGB distance between pixels is uniform and equal to $d_{rgb}$. Therefore the RGB Kernel is constant with value 
\begin{equation}
\theta_{rgb} = e^{\dfrac{-d_{rgb}^2}{2 \sigma_{rgb}}} \;.
\end{equation}

We consider the case where we have only two possible labels and the same unary potential on all the variables. Even if this assumption sounds strong, we can expect them to be locally valid. Formally, on a $N \times N$ dense grid, the energy function is defined as
\begin{align*}
E(\bx) & =  \dfrac{\Gamma \theta_{rgb} }{2\pi \sigma^2} \sum \limits_{(i,j),(i^\prime,j^\prime)}  \mathbbm{1}[\bx_{(i,j)} \neq \bx_{(i^\prime,j^\prime)} ]  e^{-\dfrac{\|(i,j)-(i^\prime,j^\prime)\|^2}{2 \sigma}} \\
& +  \sum \limits_{(i,j)} U_{(i,j)} \mathbbm{1}[\bx_{(i,j)} =0]\;,
\end{align*}
where $\sigma$ controls the range of the correlations and $U_{(i,j)}$ is a unary potential.

Since that we assumed that all the variables receive the same unary $U$, all the variables are undiscernibles.  Furthermore, the pairwise potentials are attractive, we  therefore expect all the mean-field parameters $q_{i,j} = Q(\bx_{i,j} = 0 )$ to have the same value at the fixed point solution of the Mean-Field. Therefore, we designate this common parameter $q^T$ and we can try to find analytically the Mean-Field fixed point for $q^T$ corresponding to a temperature $T$. 

At convergence, the parameter $q^T$ will have to satisfy
\begin{align*}
\log(q^T) &= \mathbb{E}_Q(E(\bx) | x_i =0) \\
& =   - \dfrac{\Gamma \theta_{rgb} }{2\pi \sigma^2T} \sum \limits_{(i,j) \in \mathbb{Z} \times \mathbb{Z}}  (1-q^T) e^{-\dfrac{\|(i,j)\|^2}{2 \sigma}} - \dfrac{U}{T}\\
& = - \dfrac{(1-q^T)\Gamma \theta_{rgb}  + U }{T}
\end{align*}
Hence, we obtain the fixed point equation
\begin{equation}
\label{method:eq:fixed_point}
\tilde{q}^T= \dfrac{1}{2}  \tanh \left({\dfrac{\tilde{q}^T\Gamma \theta_{rgb}  -U }{T}} \right)\;,
\end{equation}
where $\tilde{q}^T = q^T -0.5$. As depicted in Figure~\ref{fig:method:tanh}, when unaries are 0 (on the left) there are two distinct regimes for the solutions of this equation. For high $T$, there is only one stable solution at $\tilde{q} =0$. For low $T$, there are two distinct stable solutions where $\tilde{q} $ is close to $-0.5$ or $0.5$. The temperature threshold $T_c$ where the transition happens, corresponds to the solution of
\begin{equation}
\dfrac{1}{2}  \dfrac{d \tanh(\dfrac{\tilde{q}\Gamma \theta_{rgb} }{T})}{d \tilde{q}}|_{\tilde{q}=0} = 1 \;,
 \end{equation}
and hence $T_c = \dfrac{\Gamma \theta_{rgb} }{2}$. For real images, we have $\theta_{rgb} \leq 1$, and therefore, $T_c = \dfrac{\Gamma }{2}$  can be used to upper-bound the true critical temperature.

When unaries are non-zero, there is no closed form solution for $T_c$, however, from Equation~\ref{method:eq:fixed_point}, we can show that the smaller the unaries ($U$), the lower the critical temperature will be. This is intuitively justified in Fig.~\ref{fig:method:tanh}. 

The authors of~\cite{Weller15}, use several heuristics which basically consist in looking for high correlations and low unaries directly in the potentials of the graphical model, in order to find good variables to clamp. We, instead use a criterium based on the critical temperature in order to spot these.

\begin{figure}[ht!]
\begin{center}
\begin{tabular}{@{}cc}
\includegraphics[width=0.5\textwidth]{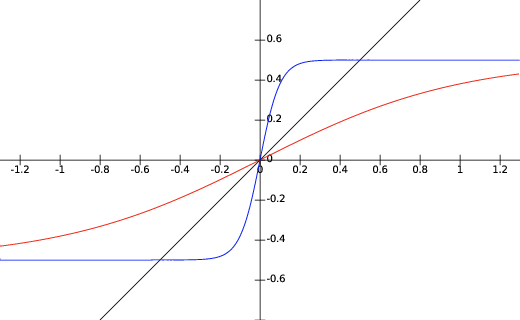} &
\includegraphics[width=0.5\textwidth]{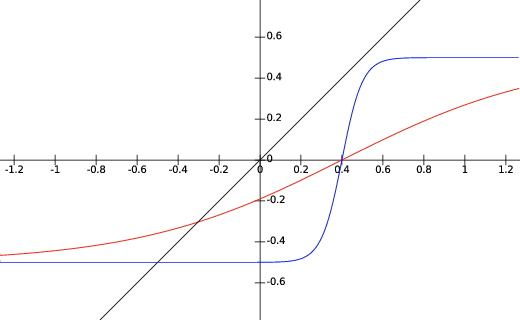} \\
Without unaries & With unaries
\end{tabular}
\end{center}
   \caption{ $\tanh(\dfrac{\tilde{q}\Gamma - U}{T})$ for two temperatures. Low T (blue) and High T (red).}
\label{fig:method:tanh}
\end{figure}


\subsection{Experimental analysis}

We use the dense CRF implementation of~\cite{Kraehenbuehl13} to verify the phase transition experimentally for $\Gamma =10$. In our experiments, we used the three following settings, which range from the stylised example used for calculation to real semantic segmentation problems:
\begin{itemize}
\item {\bf Model 1: } We use a uniform rgb image $d_{rgb} = 0$. Two classes without unary potentials. This is exactly the model used for the derivations with $\theta_{rgb} = 1$ and $U = 0$. 
\item {\bf Model 2: } Gaussian potentials defined over image coordinates distance + RGB distance. Two classes without unary potentials. In other words, $\theta_{rgb} \leq 1$. 
\item {\bf Model 3: }Gaussian potentials defined over image coordinates distance + RGB distance. Two classes with unary potentials produced by a CNN. This is a real-life scenario.
\end{itemize}

Fig.~\ref{fig:graph_transition} shows that, as expected, two regimes appear for Model 1, before and after $T = 5$. We see that our prediction remains completely valid for Model 2, some non-uniform regions fall under the regime $\theta_{rgb} \leq 1$ and therefore the 10 \% highest entropy percentile transitions slightly earlier. For Model 3, however, we see that the minimal and average entropy remain low even for $T > 5$. This is well explained by the fact that large regions of the image receive strong unary potentials from one class or the other, and therefore fall under the case "with unaries" of Fig.~\ref{fig:method:tanh} where the $U$ parameter cannot be ignored. However, some uncertain regions receive unary potentials of same value for both labels, and therefore undergo a phase transition as predicted by our calculation. That is why the maximal entropy behaves similarly to Model 2. Our algorithm precisely targets these uncertain regions.

Interestingly, we see that in practice, the users of DenseCRF choose the $\Gamma$ and $T$ parameters in order to be in a Multi-Modal regime, but close to the phase transition. For instance in the public releases of~\cite{Chen15b} and ~\cite{Zheng15}, the Gaussian kernel is set with $T = 1$ and $\Gamma = 3$.
\PB{Explain why this is an upper-bound.}

\begin{figure}[ht!]
\begin{center}
\begin{tabular}{@{}ccc}
\includegraphics[width=0.4\textwidth]{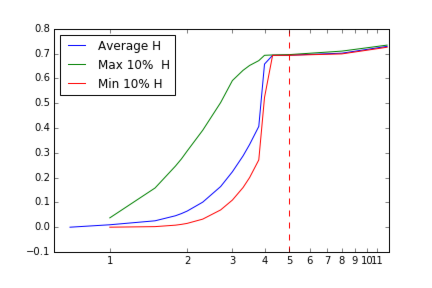} &
\includegraphics[width=0.4\textwidth]{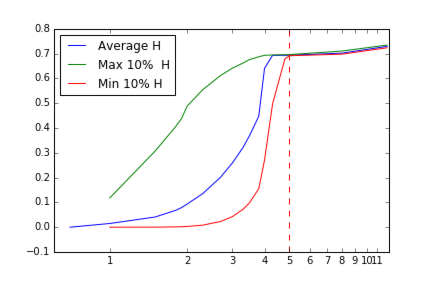} &
\includegraphics[width=0.4\textwidth]{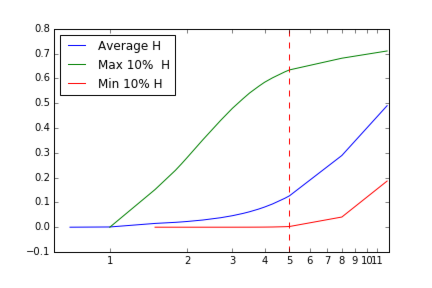} \\
Model 1 & Model 2 & Model 3
\end{tabular}
\end{center}
   \caption{ Entropy as a function of temperature.}
\label{fig:graph_transition}
\end{figure}

\begin{figure}[ht!]
\begin{center}
\includegraphics[width=0.3\textwidth]{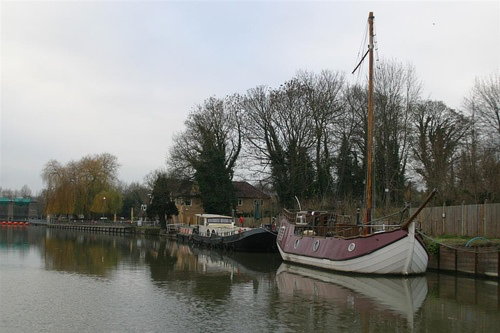} \\
RGB Image \\
\includegraphics[width=1.2\textwidth]{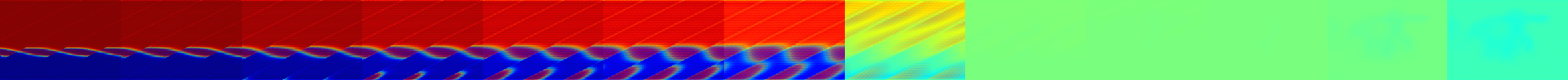} \\
Model 1 \\
\includegraphics[width=1.2\textwidth]{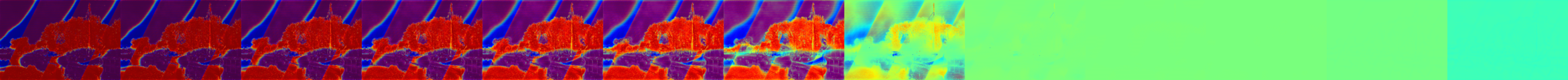} \\
Model 2 \\
\includegraphics[width=1.2\textwidth]{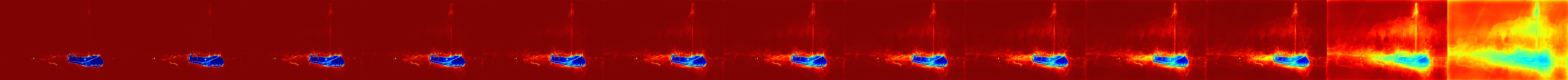} \\
Model 3 \\
\end{center}
\caption{Evolution of MF probability for background label when temperature increases}
\label{fig:MMKSP}
\end{figure}

\section{K-Shortest Path algorithm for the Multi-Modal Probabilistic Occupancy Maps}

We present here the algorithm we use to reconstruct tracks from the Multi-Modal Probabilistic Occupancy Maps (MMPOMs) of Section 7.2.

\paragraph{KSP}  
In  the original  algorithm of~\cite{Berclaz11}, the  POMs
computed at  successive instants  were used  to produce  consistent trajectories
using  the a  K-Shortest  Path (KSP)  algorithm~\cite{Suurballe74}. This  involves
building a graph in which each ground  location at each time step corresponds to
a node and  neighboring locations at consecutive time steps  are connected.  KSP
then finds a set of node-disjoint shortest paths in this graph where the cost of
going through a location is proportional  to the negative log-probability of the
location in  the POM~\cite{Berclaz11}. The KSP problem can be solved in linear time 
and an efficient implementation is available online.

\paragraph{KSP for Multi-Modal POM} Since MMMF produces  multiple POMs, one
for each  mode, at each  time-step, we duplicate the  KSP graph nodes,  once for
each mode as well.   Each node is then  connected to each copy  of neighboring locations
from previous and  following time steps. We then solve  a multiple shortest-path
problem in this new graph, with the additional constraint that at each time step
all the paths have  to go through copies of the nodes  corresponding to the same
mode.   This larger  problem is  NP-Hard and  cannot be  solved by  a polynomial
algorithm such as  KSP.  We therefore use the Gurobi Mixed-Integer Linear Program solver~\cite{Gurobi}.

More precisely, let us assume that we have a sequence of Multi-Modal POMs $Q^t_{k}$ and mode probabilities $m^t_k$ for $t \in \{1,\dots,T\}$ representing time-steps and $k \in \{1,\dots,K\}$ representing different modes. Each $Q^t_{k}$ is materialized through a vector of probabilities of presence $q^t_{k,i}$, where each $i \leq N$ is indexes a location on the tracking grid.

Using the grid topology, we define a neighborhood around each variable, which corresponds to the maximal distance a walking person can make on a grid in one time step. Let us denote by $\mN_i$ the set of indices corresponding to locations in the neighbourhood of $i$. The topology is fixed and hence $\mN_i$ does not depend on the time steps. We define the following log-likelihood costs.

Using a Log-Likelihood penalty, we define the following costs:
\begin{itemize}
\item $C^t_{k,i} = \log \left(\dfrac{1 - q^t_{k,i}}{q^t_{k,i}} \right)$, representing the cost of going through variable $i$ at time $t$ if mode $k$ is chosen.
\item $C^t_{k} = \log \left(\dfrac{1 - m^t_k}{m^t_k} \right)$, representing the cost of choosing mode $k$ at time $t$.
\end{itemize}

We solve for an optimization problem involving the following variables:
\begin{itemize}
\item $x^{t}_{k,i,l,j}$ is a binary flow variable that should be $1$ if a person was located in $i$ at $t$ and moved to $j$ at $t+1$, while modes $k$ and $l$ were respectively chosen at time $t$ and $t+1$.
\item $y^t_k$ is a binary variable that indicates whether mode $k$ is selected at time $t$.
\end{itemize}

We can then rewrite the Multi-Modal K-Shortest Path problem as the following program, were we always assume that $t \leq T$ stands for a time step, $k \leq K$ and  $l \leq K$ stand for mode indices, and $i \leq N$ and  $j \leq N$ stand for grid locations:

\begin{equation}
\begin{aligned}
\label{eq:MILP}
&  \min  
& &   \sum \limits_{t , k } C^t_k  y^t_k + \sum \limits_{t , k,  l \leq K}  \sum \limits_{i, j \in \mN_i}  C^t_{k,i} x^{t}_{k,i,l,j}\\
& \text{s.t.} & &  \forall (t , k, i)\;, \; \; \sum \limits_{l, j\in \mN_i} x^{t-1}_{l,j,k,i} = \sum \limits_{l, j\in \mN_i} x^{t}_{k,i,l,j} & &\texttt{ flow conservation}\; \\
& \text{ } & &  \forall (t , k, i)\;, \; \;\sum \limits_{l, j\in \mN_i} x^{t}_{k,i,l,j} \leq y^t_k & & \texttt{ disjoint paths + selected mode}\; \\
& \text{ } & &  \forall t \;,\; \; \sum \limits_{k} y^{t}_{k} = 1 & & \texttt{ selecting one mode}\;\\
& \text{ } & &  \forall t,k,i,l,j  \;,\; \;  0 \leq x^{t}_{k,i,l,j} \leq 1\;\\
& \text{ } & &  \forall t,k \;,\; \;  y^{t}_{k}  \in \{0,1\}\;
\end{aligned}
\end{equation}

\begin{figure}[ht!]
\begin{center}
\includegraphics[width=0.90\textwidth]{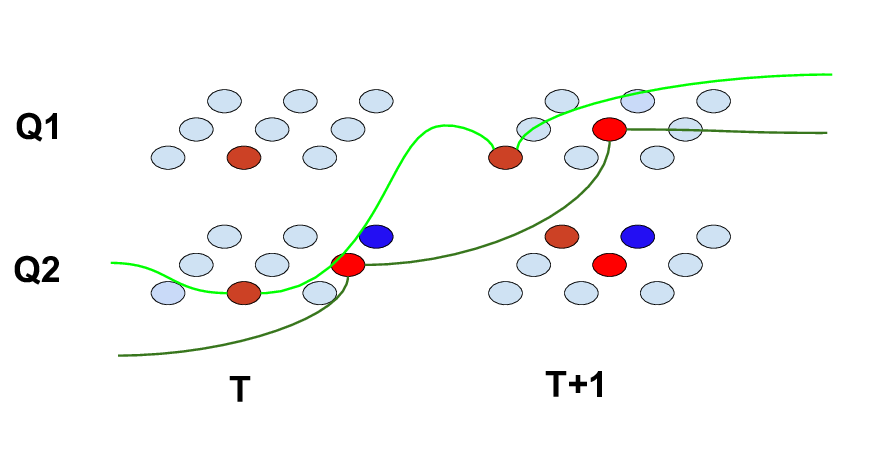} \\
\end{center}
\caption{Illustration of the output of our K-Shortest Path algorithm in the case of multiple modes.}
\label{fig:MMKSP}
\end{figure}

\paragraph{KSP prunning}
However, the problem as written above, may involve several tens millions of flow variables and therefore becomes intractable, even for the best MILP solvers. We therefore first prune the graph to drastically reduce its size.

The obvious strategy would be by thresholding the POMs and removing all the outgoing and incoming edges from locations which have probabilities below $q_{thresh}$. However, this would be self-defeating as one of the main strengths of the KSP formulation is to be very robust to missing-detections and be able to reconstruct a track even if a detection is completely lost for several frames. 

We therefore resort to a different strategy. More precisely, we initially relax the constraint \texttt{ disjoint paths + selected mode}, to a simple disjoint path constraint, and remove the constraint \texttt{ selecting one mode}. We therefore obtain a relaxed problem
\begin{equation}
\begin{aligned}
\label{eq:MILP_relax}
&  \min  
& &   \sum \limits_{t , k }  \sum \limits_{t , k,  l \leq K}  \sum \limits_{i, j \in \mN_i}  C^t_{k,i} x^{t}_{k,i,l,j}\\
& \text{s.t.} & &  \forall (t , k, i)\;, \; \; \sum \limits_{l, j\in \mN_i} x^{t-1}_{l,j,k,i} = \sum \limits_{l, j\in \mN_i} x^{t}_{k,i,l,j} & &\texttt{ flow conservation}\; \\
& \text{ } & &  \forall (t , k, i)\;, \; \;\sum \limits_{l, j\in \mN_i} x^{t}_{k,i,l,j} \leq1 & & \texttt{ disjoint paths}\; \\  
& \text{ } & &  \forall t,k,i,l,j  \;,\; \;  0 \leq x^{t}_{k,i,l,j} \leq 1\;
\end{aligned}
\end{equation}

which is nothing but a vanilla K-Shortest Path Problem. It can be solved using our linear-time KSP algorithm. This KSP problem will output a very large number of paths, going through all the different modes simultaneously. From, this output, we extract the set of grid locations which are used, in any mode, at each time step, and select them as our potential locations in the final program. In our current implementation, we add to these locations, the ones for which $q^t_{k,i} \geq q_{thresh}$ for any mode at time-step $t$. 

We can finally solve Program~\ref{eq:MILP}, where non-selected locations are pruned from the flow graph. We don't know if our strategy, based on a relaxation and pruning, provides a guaranteed optimal solution to~\ref{eq:MILP}, but this is an interesting question.

\section{Pseudo-code for the Multi-Modal Mean-Fields algorithm}

Algorithm~\ref{algo:Split} summarises the operations to split one mode into two, or, in other words, to obtain the two additional constraints which are used to define the two newly created subsets.
Algorithm~\ref{algo:MMMF} summarises the operations to obtain the Multi-Modal Mean Field Distribution by constructing the whole Tree.

In Algorithm~\ref{algo:MMMF}, $ConstraintTree$, is taken to be a Tree in the form of a list of constraints, one for each branching-point, or leaf,---except for the root---, in a breadth first order. The function $ \texttt{pathto}(nNode)$, returns the set of indices corresponding to the branching points on the path to the branching point, or leaf with index  $nNode$, including index $nNode$ itself.

\begin{algorithm}
\label{algo:Split}
\caption{Function:$\texttt{Split}(ConstraintList)$}
\textbf{Input:}\\
$E(\bx)$: An Energy function defined by a CRF;\\
$\texttt{SolveMF}(E,ConstraintList)$: A Mean Field solver with cardinality constraint.;\\
$Temperatures$: A list of temperatures in increasing order;\\
$\mH_{low},\mH_{high}$: Entropy thresholds for the phase transition. 0.3 and 0.6 here.\\
$C$: A cardinality threshold\\
\textbf{Output:} \\
$LeftConstraints$: A triplet containing a list of variables, clamped to value, -C \\
$RightConstraints$: A triplet containing a list of variables, clamped to value, C\\
\begin{algorithmic}
\STATE{$Q^{T_0} \leftarrow \texttt{SolveMF}(E)$}
\FOR{$\texttt{T in }Temperatures$}
	\STATE{$Q^T \leftarrow \texttt{SolveMF}(\dfrac{E}{T},ConstraintList)$}
	\STATE{$i_{list} \leftarrow [.]$}
	\STATE{$v_{list} \leftarrow [.]$}
	\FOR{\texttt{index in $1\dots \texttt{len}(Q^t)$, v in }$labels$}
		\IF{$\mathbbm{1}[\mH(q^{T}_{index}) > 0.6]\mathbbm{1}[\mH(q^{T_0}_{index})<0.3]\mathbbm{1}[q^{T_0}_{index,v} > 0.5] = 1$}
			\STATE{$i_{list}$\texttt{.append(index)},$ v_{list}$\texttt{.append(v)}}
		\ENDIF
	\ENDFOR
	\IF{$\texttt{len}(i_{list}) >0$}
	\STATE{$\texttt{exit for loop}$}
	\ENDIF
\ENDFOR
	\STATE{$LeftConstraints$ = $i_{list}, v_{list},-C$}
	\STATE{$RightConstraints$ = $i_{list}, v_{list},C$}
\RETURN{$LeftConstraints$,$RightConstraints$}
\end{algorithmic}
\end{algorithm}

\begin{algorithm}
\caption{Compute Multi-Modal Mean Field}
\label{algo:MMMF}
\textbf{Input:}\\
$E(\bx)$: An Energy function defined on a CRF;\\
$\texttt{SolveMF}(E,ConstraintList)$: A Mean Field solver with cardinality constraint;\\
$\texttt{Split}(ConstraintList)$: Alg.~\ref{algo:Split}. A function that computes the new constraints. \\
$NModes$: A target for the number of modes in the Multi-Modal Mean Field\\
\textbf{Output:} \\
$Qlist$: A list of Mean Field distributions in the form of a table of marginals \\
$mlist$: A list of probabilities, one for each mode\\
\begin{algorithmic} 
\STATE{$ConstraintTree =[.]$} \\
\text{We first build the tree by adding constraints.} 
\WHILE{$nNode <NModes$}
	\STATE{$ConstraintList =[.]$}
	\FOR{$p  \texttt{ in pathto}(nNode)$}
		\STATE{\texttt{$ConstraintList$.append(ConstraintTree[p])}}
	\ENDFOR
	\STATE{$LeftConstraints,RightConstraints \leftarrow \texttt{Split}(ConstraintList)$}
	\STATE{$ConstraintTree$.\texttt{append}($LeftConstraints$)}
	\STATE{$ConstraintTree$.\texttt{append}($RightConstraints$)}
\ENDWHILE \\
\text{We now turn to the computation of on MF distribution per leaf.}
\STATE{$Qlist=[.],Zlist =[.], mlist =[.]$}
\FOR{mode in $0 \dots NModes$}
	\STATE{$ConstraintList =[.]$}
	\FOR{$p  \texttt{ in pathto}(mode + NModes -1)$}
		\STATE{\texttt{$ConstraintList$.\texttt{append}(ConstraintTree[p])}}
	\ENDFOR
	\STATE Q,Z $\leftarrow$ \texttt{SolveMF}(E,ConstraintList)
	\STATE{$Qlist.\texttt{append}(Q)$}
	\STATE{$Zlist.\texttt{append}(Z)$}
\ENDFOR \\
\text{Finally, we compute the mode probabilities.} 
\FOR{$mode$ in $0 \dots NModes$}
	\STATE{$mlist.\texttt{append}(\dfrac{Zlist[mode]}{\sum Zlist})$}
\ENDFOR
\RETURN{$Qlist$, $mlist$}

\end{algorithmic}
\end{algorithm}

\end{appendices}

\end{document}